\documentclass{article}

\usepackage{amsmath}
\usepackage{amsfonts}
\usepackage{amsthm}
\usepackage{amssymb}
\usepackage{color}
\usepackage{tikz}
\usepackage{enumitem}
\usepackage{bbm}

\usetikzlibrary{arrows,decorations.pathmorphing,backgrounds,positioning,fit,petri,calc,shadows, decorations.pathreplacing}

\newtheorem{theorem}{Theorem}[section]
\newtheorem{lemma}[theorem]{Lemma}
\newtheorem{definition}[theorem]{Definition}
\newtheorem{corollary}[theorem]{Corollary}

\newtheorem{remark}[theorem]{Remark}

\DeclareMathOperator{\relu}{ReLu}
\DeclareMathOperator{\rank}{rank}
\DeclareMathOperator{\sign}{sign}

\newcommand{\R}{\mathbb{R}}

\usepackage{colortbl}
\definecolor{DarkBlue}{rgb}{0,0,0.7} 
\long\def\comment#1{}

\newcommand{\blue}[1]{\textcolor{blue}{#1}}
\newcommand{\rh}[1]{{\bf{{\blue{{RH --- #1}}}}}}

\usepackage[backend=biber]{biblatex}
\newcommand{\footremember}[2]{
	\footnote{#2}
	\newcounter{#1}
	\setcounter{#1}{\value{footnote}}
}

\author{Stefan Bamberger\footremember{holsten}{stefan.bamberger@tum.de, Department of Mathematics, Technical University of Munich, now Holsten Systems GmbH}
        \and
        Reinhard Heckel\footremember{tum1}{reinhard.heckel@tum.de, Department of Computer Engineering, Technical University of Munich and Munich Center for Machine Learning}
        \and 
        Felix Krahmer\footremember{tum2}{felix.krahmer@tum.de, Department of Mathematics and Munich Data Science Institute, Technical University of Munich and Munich Center for Machine Learning}}

\begin{document}
	
	\title{Approximating Positive Homogeneous Functions with Scale Invariant Neural Networks}
	
	\maketitle
	
	\begin{abstract}
		We investigate to what extent it is possible to solve linear inverse problems with $\relu$ networks. 
    Due to the scaling invariance arising from the linearity, an optimal reconstruction function $f$ for such a problem is positive homogeneous, i.e., satisfies $f(\lambda x) = \lambda f(x)$ for all non-negative $\lambda$. In a $\relu$ network, this condition translates to considering networks without bias terms. 
    We first consider recovery of sparse vectors from few linear measurements. We prove that $\relu$- networks with only one hidden layer cannot even recover $1$-sparse vectors, not even approximately, and regardless of the width of the network. However, with two hidden layers, approximate recovery with arbitrary precision and arbitrary sparsity level $s$ is possible in a stable way. We then extend our results to a wider class of recovery problems including low-rank matrix recovery and phase retrieval.
		Furthermore, we also consider the approximation of general positive homogeneous functions with neural networks. Extending previous work, we establish new results explaining under which conditions such functions can be approximated with neural networks.
  Our results also shed some light on the seeming contradiction between previous works showing that neural networks for inverse problems typically have very large Lipschitz constants, but still perform very well also for adversarial noise. Namely, the error bounds in our expressivity results include a combination of a small constant term and a term that is linear in the noise level, indicating that robustness issues may occur only for very small noise levels.
	\end{abstract}
	
	\section{Introduction}
A variety of signal reconstruction problems including accelerated magnetic resonance imaging and accelerated computed tomography can be formulated as reconstructing a signal from few measurements. Traditionally, reconstruction  is often formulated as assuming the signal is sparse in some basis~\cite{carota06-1, Don06} and reconstructing a sparse vector from few measurements. This is known as compressed sensing~\cite{carota06-1, Don06}. 
A sparse signal is typically reconstructed by solving a convex optimization problem.

The most commonly studied setup in compressed sensing is to recover a signal $x \in \mathbb{R}^n$ from $m$ linear measurements $y = A x \in \mathbb{R}^m$ ($A \in \mathbb{R}^{m \times n}$), assuming that $x$ is $s$-sparse, i.e., at most $s$ of its $n$ entries are non-zero. For suitable measurement matrices $A$, for example random matrices with Gaussian entries, compressed sensing can provide recovery guarantees even if the number of measurements $m$ is significantly smaller than the dimension $n$. In addition, $\ell_1$ minimization is an efficient method to find the unique solution. For details on the theory and the algorithms see \cite{comp_sen}.
 
Neural networks outperform classical sparsity-based methods for a variety of signal and image reconstruction problems. 
Neural networks achieve state-of-the-art results in tasks such as denoising \cite{zhang_GaussianDenoiserResidual_2017} and reconstructing images from few and noisy examples~\cite{zbontar_FastMRIOpenDataset_2018}. 
However, contrary to optimization-based methods, for which a rich literature on performance guarantees exists~\cite{comp_sen}, many underlying theoretical questions are still open for neural network-based signal reconstruction.

 In this work, we consider the question of whether a sparse signal can be provably recovered with a neural network. 
	We study neural networks $f$ that reconstruct a sparse signal $x$ from a measurement $y = A x$, i.e., $x = f(y)$ or at least ensure that the reconstruction error $\|x - f(y)\|$ is small. 
	
We investigate under what conditions (in particular number of required layers) bias-free $\relu$ networks can approximately solve the sparse recovery problem. 

We consider bias free $\relu$ networks since they make use of the positive homogeneous structure of the problem. 
Specifically, we want the reconstruction network $f$ to satisfy $f(\lambda y) = \lambda f(y)$ for all $\lambda \geq 0$ because we know that if $x$ has measurement $y = A x$, then the measurement $\lambda y$ will be obtained from the signal $\lambda x$. 
Such a scale-invariant network has the property that if it can reconstruct every $s$-sparse vector on the unit-ball then it can also recover each $s$-sparse vector on an unbounded domain. Therefore investigating the number of required layers goes beyond the usual universal approximation theorem (\cite{pinkus_1999}, see Theorem~\ref{thm:universal_approximation} below), which can guarantee arbitrarily precise approximations but only on a compact domain and without incorporating the positive homogeneous structure of the problem into the network.

    Taking knowledge about a function into account for the design of the network to approximate it is a strategy that can significantly improve reliability and training effort. For this reason, a large number of works have studied this strategy over a long period of time for different types of functions \cite{wood1996representation, dugas2009incorporating, kicki2021, Chidester2019}. Specifically, Tang et al.~\cite{tang2020towards} consider positive homogeneous functions.

	Furthermore, positive homogeneity can also have applications for other problems. For example in image denoising, rescaling the brightness of a picture might not change the underlying procedure and if the corresponding network is designed to be positive homogeneous, different brightness levels do not need to be learned separately.

	\subsection{Contributions of this work} 
	\label{sec:goals_this_work}

	We first show that with one hidden layer, it is not possible to even approximately recover $1$-sparse vectors. 
    Secondly, we show that two hidden layers are sufficient to recover sparse vectors with arbitrary sparsity levels $s$ and to arbitrary precision. We also establish a robustness guarantee for these networks ensuring that the resulting network reconstruction function can reconstruct vectors that are approximately $s$-sparse from corrupted measurements.
	
We also show that our reconstruction guarantee generalizes as follows to a more general class of inverse problems. Instead of the set of sparse vectors, we can have any subset $U \in \R^n$ that is positive homogeneous (i.e., $\lambda u \in U$ for any $\lambda \in [0, \infty)$ and $u \in U$), and instead of a linear map $A$ we can have a positive homogeneous function $g$ that satisfies certain requirements. Besides sparse recovery, using this generalized result we also show that low-rank matrix recovery and phase retrieval problems can be solved using $\relu$ networks with two hidden layers.
	
	Since for problems like sparse recovery or phase retrieval, there are reconstruction methods based on optimization problems, we also show  how these can be translated to a $\relu$ network with two hidden layers. The central argument is that there exists a continuous solution function of the optimization problem. We show this using a generalization of the continuity concept to functions with multiple values.
	
	Furthermore, we also gain more insights into the general approximation of continuous, positive homogeneous functions. Specifically, \cite{tang2020towards} already shows that with the $\relu$ activation function, the unbiased networks with two hidden layers represent a class of functions such that (i) all these functions are positive homogeneous and (ii) they can approximate every continuous positive homogeneous function to arbitrary precision. We complement their result by showing that (up to certain modifications of itself), the $\relu$ function is the only activation function such that the unbiased networks satisfy these two conditions (i) and (ii). We establish this in a theorem that is similar to the classical universal approximation theorem of neural networks. Furthermore, using the negative results about sparse recovery, we also prove that this universal approximation property fails to hold for just one hidden layer such that the assumption of two hidden layers is actually necessary.

	In Section~\ref{sec:main_results}, we present our main results in two parts. One part is about solving inverse problems and the other one about universal approximation of positive homogeneous functions. Then Section~\ref{sec:universal_approximation} contains the proof of the main result about universal approximation. In Section~\ref{sec:inverse_problems}, we prove the main results regarding inverse problems and show some other applications of them. Section~\ref{sec:optimization_networks} then shows that $\relu$ networks can be used to solve inverse problems in the way optimization-based methods do.  In Section~\ref{sec:nn_discussion}, we discuss implications and relations to other work.

\subsection{Related work}

We briefly discuss prior work on universal approximation of neural networks, prior work on sparse recovery with neural networks, and prior work on robustness. 

\paragraph{Universal approximation with neural networks.}

	To solve the aforementioned signal recovery problem, we need to compute the function that maps measurements $y = A x$ to their original signals $x$. Compressed sensing guarantees the well-definedness of this function and the question is if, how, and how well this function can be approximated by certain classes of neural networks.
	
	The general question of how well certain functions can be approximated has been a central question in the research of neural networks for a long time. Cybenko \cite{cybenko_uat} showed that neural networks with only one hidden layer and any bounded measurable sigmoidal activation function can approximate any continuous function on the $n$-dimensional unit cube to arbitrary precision if the width of the network is sufficiently large. This result has been known as the \textit{universal approximation theorem} and has been extended several times. For example, Leshno et al.~\cite{leshno_uat} generalized it to the case of any non-polynomial activation function. \cite{pinkus_1999} even proved for a large class of functions that this approximation property is equivalent to the function being non-polynomial.

    \begin{theorem} [Universal Approximation Theorem, \cite{pinkus_1999}] \label{thm:universal_approximation}
		Let $n \geq 1$ be a dimension and $\sigma: \R \rightarrow \R$ continuous. Then the following are equivalent.
		\begin{enumerate}[label=(\alph*)]
			\item For any compact $K \subset \R^n$, any continuous $f: K \rightarrow \R$ and any $\delta > 0$, there exists a network with one hidden layer and activation function $\sigma$, representing $\tilde{f}: \R^n \rightarrow \R$, such that for all $x \in K$,
			\[
			|\tilde{f}(x) - f(x)| \leq \delta.
			\]
			
			\item $\sigma$ is not a polynomial.
		\end{enumerate}
	\end{theorem}
	
	As mentioned previously, taking known properties of the approximated functions into account for the network design has been studied in multiple previous works.	To mention some concrete examples, \cite{dugas2009incorporating} considers functions satisfying a certain monotonicity and convexity condition, \cite{kicki2021} considers functions that are invariant under certain permutations of their input variables, and \cite{tang2020towards} considers the positive homogeneous functions.
	
	All these works construct a class of networks and show that, on the one hand, all these networks represent functions of the particular class, and on the other hand, every function of the respective class can be approximated by one of these networks. The latter part corresponds to the implication (b) $\Rightarrow$ (a) in the universal approximation Theorem~\ref{thm:universal_approximation} for the class of general continuous functions on compact domains.
	
	In particular Tang et al. \cite{tang2020towards} show these things for positive homogeneous functions. In this work, we extend their result to an equivalence statement similar to Theorem~\ref{thm:universal_approximation} which will be Theorem~\ref{thm:uat_homogeneous}. We also show that their requirement of having at least two hidden layers is required.

	\paragraph{Sparse recovery with neural networks. }
	A popular approach to recovery a sparse vector $x^\ast$ from a linear measurement $y=Ax^\ast$ is to solve the $\ell_1$ regularized least-squared problem
	\begin{align}
		\min_{x} \frac{1}{2} \| Ax - y \|_2^2 + \lambda \|x\|_1.
	\end{align}
	This is a convex program and popular algorithms to solve it are proximal methods, such as the iterative shrinkage thresholding algorithm (ISTA). ISTA is initialized at some $x^0$ and iterates for $\ell=1,2,\ldots$
	\begin{align}
		x^{\ell+1} = \eta_{\lambda/L}\left(
		x^\ell - \frac{1}{L} A^T (A x^\ell - y)
		\right),
	\end{align}
	where $\eta_z$ is the soft-thresholding function, i.e., $\eta_z(t) = \sign(t)(|t| - z)$ if $|t| \geq z$ and $\eta_z = 0$ otherwise. Note that $d$ many iteration can be viewed as a recurrent neural network with depth $d$. Inspired by this, a number of works starting with~\cite{gregor_LearningFastApproximations_2010}
	studied un-rolled algorithms which unrolls $d$ many ISTA iterations as 
	\begin{align}
		x^{\ell+1} = \eta_{\lambda/L}\left(W_1^\ell
		x^\ell + W_2^\ell y
		\right).
	\end{align}
	This is a feed-forward neural network of depth $d$ and $W_1^\ell$ and $W_2^\ell$ are weight matrices that are typically learned based on data. In all layers, $\eta_{\lambda/L}$ is used as an activation function.
	Chen at al.~\cite{chen_TheoreticalLinearConvergence_2018} (Thm. 2) established that there exist choices of weights such that a $s$-sparse signal $x^\ast$ with entries bounded by $|x_i^\ast| \leq B$, and with $s$ sufficiently small can be approximated as
	\begin{align*}
		\| x^d - x^\ast \|_2 \leq s B e^{-cd},
	\end{align*}
	where $c$ is a constant depending on the matrix $A$ and mildly on the sparsity $s$ of the signal. This result requires at least $s^2\leq m$, where $m$ is the number of measurements, as it works with the incoherence of the matrix $A$, and requires the sparsity to be sufficiently small relative to the incoherence (see~\cite[Appendix B, Step 3]{chen_TheoreticalLinearConvergence_2018}. 
	This result establishes that there is a relatively shallow neural network that can approximately sparse signals well.

	So with those approaches, depth $d = \mathcal{O}(\log s)$ is sufficient to approximate the signal $x^*$ with the output of the network. In contrast, the goal of this work is to determine the exact number of layers that is necessary and sufficient to solve the sparse recovery problem. Furthermore, other than the aforementioned unrolling approach, which works for signals whose entries are bounded by $B$, we consider networks that can solve the sparse recovery problem on the entire (unbounded) set of possible signals. We prove that for these requirements and with $\relu$ activation function, one hidden layer is not sufficient to solve the problem, but two hidden layers are, even for a stable solution. However, our results do not yield a construction for these networks and also do not specify their width.
	
\paragraph{Robustness.} 
\label{subsec:nn_robustness_intro}

	An important aspect of solutions to the sparse recovery problem is how sensitive they are to noisy measurements $y = A x + e$. For a robust recovery method $f: \R^m \rightarrow \R^n$, we expect $\|f(A x + e) - f(A x)\|_2$ to be small for small $\|e\|_2$. With minimization-based methods such as the quadratically-constrained basis pursuit
	\begin{align} \label{eq:qc_bp_nn_introduction}
	    & \min_{x} \|x\|_1 & \text{s.t. } & \|A x - y\|_2 \leq \eta
	\end{align}
	for a parameter $\eta \geq 0$, robust recovery has been proven to be successful (see Section~4.3 in \cite{comp_sen}).
	
	For neural networks on inverse problems, the question of robustness is currently studied under various aspects. In \cite{genzel2022solving}, an empirical analysis is conducted suggesting that neural networks can provide robust solutions to specifically chosen problems similar to sparse recovery and image reconstruction in a similar or even better way compared to optimization-based methods. 

 In practice, neural networks are relatively stable to worst-case perturbation, empirical studies indicate that they are similarlty sensitive or stable to worst-case perturbations than traditional methods~\cite{Darestani_Chaudhari_Heckel_2021}. 
 
	In contrast to this, \cite{gottschling2020troublesome} provides a theoretical analysis of certain situations that necessarily lead to robustness issues for neural networks on inverse problems. Specifically, they show that in certain situations, neural networks applied to inverse problems necessarily have large local Lipschitz constants.

   Krainovic~\cite{krainovic_learning_2023} discusses jittering and robust training to provably learn robust denoisers. 
	
	We also review our work in the context of the aforementioned results. This provides a possible interpretation of this seeming contradiction. We show robustness similar to the one for the minimization problem \eqref{eq:qc_bp_nn_introduction} but still the local Lipschitz constants of our solution might be very large. This is due to possible large gradients $\frac{\|f(A x + e) - f(A x)\|_2}{\|e\|_2}$ for very small error levels $\|e\|_2$. This is discussed in detail in Section~\ref{sec:nn_robust}.

	\subsection{Notation}
	
	We consider neural networks with the activation function $\relu: \R \rightarrow \R$, defined by $\relu(x) = \max\{x, 0\}$. We also use the shorter notation $\phi := \relu$.
	
	For $x \in \R^n$, we denote $\|x\|_0$ for the number of non-zero entries. We write the set of $s$-sparse vectors as $\Sigma_s := \{x \in \R^n:\quad \|x\|_0 \leq s\}$.
	
	A lot of statements in this work concern feedforward neural networks of the following type.
	\begin{definition}
	    We say that a function $f: \R^m \rightarrow \R^n$ is represented by a (feedforward) neural network with activation function $\sigma: \R \rightarrow \R$ and $d$ hidden layers if there are matrices
	    \begin{align*}
	        W_1 \in \R^{k_1 \times m}, W_2 \in \R^{k_2 \times k_1}, \dots, W_{d} \in \R^{k_d \times k_{d - 1}}, W_{d + 1} \in \R^{n \times k_d}
	    \end{align*}
	    and vectors
	    \begin{align*}
	        b_1 \in \R^{k_1}, \dots, b_d \in \R^{k_d}, b_{d + 1} \in \R^{k_{n}}
	    \end{align*}
	    such that for all $x \in \R^m$,
	    \begin{align*}
	        f(x) = W_{d + 1} \sigma\left( W_d \sigma( \dots \sigma(W_1 x + b_1) \dots ) + b_d \right) + b_{d + 1},
	    \end{align*}
	    where $\sigma$ is applied component-wise $d$ times.
	    
	    We call $W_1, \dots, W_{d + 1}$ the weight matrices, $b_1, \dots, b_{d + 1}$ the biases of the network and say that the network is unbiased if $b_j = 0$  for $1 \leq j \leq d + 1$. $d$ is called the depth of the network and $k_1, \dots, k_d$ the widths of the hidden layers.
	    
	    If $\sigma = \relu$, then we also call the network a \textit{$\relu$ network}.
	\end{definition}
	
	\begin{definition}
		We define a set $U \subset \R^n$ to be positive homogeneous if for all $\lambda \in [0, \infty)$ and all $x \in U$, also $\lambda x \in U$.
		
		If $U \subset \R^n$ is a positive homogeneous set, we define a function $f: U \rightarrow \R^m$ to be positive homogeneous if for all $\lambda \in [0, \infty)$ and all $x \in U$, $f(\lambda x) = \lambda f(x)$.
	\end{definition}
	
	Furthermore, we use the following basic notations:
	\begin{itemize}
	    \item $[n] = \{1, \dots, n\}$ for $n \in \mathbb{Z}_{\geq 1}$.
	    \item $B_r(x_0) := \{x \in X \,\big|\, d(x, x_0) < r\} \subset X$ is the open and $\bar{B}_r(x_0) := \{x \in X \,\big|\, d(x, x_0) < r\}$ the closed ball with radius $r$ in the metric space $(X, d)$. Unless noted otherwise, if $X = \R^n$, then $d$ is the Eucledian distance.
	\end{itemize}

	\section{Main Results} 
 
 In this section we state and discuss our main results for approximating positive homogeneous functions with bias-free $\relu$-networks. 
 We start with considering the reconstruction of sparse signals from underdetermined linear measurements, then extend our results to more general constraint sets, and finally we consider the problem of approximating arbitrary positive homogeneous functions.

 \label{sec:main_results}
	
\subsection{$\relu$-networks reconstructing sparse signals}

	The main results of this work regarding sparse recovery are the following two Corollaries~\ref{cor:imposs_one_hidden} and \ref{cor:approximation_rip} which are consequences of the slightly more general Theorems~\ref{thm:approximation_inverse_function} and \ref{thm:lower_bound_generalized} below. 
	
	Corollary \ref{cor:imposs_one_hidden} below states that a $\relu$ network with one hidden layer cannot recover all sparse vectors from any $m \ll n$ linear measurements, not even approximately and for $1$-sparse vectors. 
	
	\begin{corollary}[Impossibility result for one hidden layer] \label{cor:imposs_one_hidden}
		Let $A \in \mathbb{R}^{m \times n}$, $m \leq n$, and $f: \R^m \rightarrow \R^n$ be a function represented by a $\relu$ network with one hidden layer. Then, for any width and any choice of the network parameters,
		\[
		\sup_{x \in \Sigma_1 \backslash \{0\}} \frac{\|x - f(A x)\|_2}{\|x\|_2} \geq \sqrt{1 - \frac{m}{n}},
		\]
  where $\Sigma_1$ is the set of one-sparse vectors. 
	\end{corollary}
	
	Note that in typical signal reconstruction problems, the number of measurements is much smaller than the dimension of the vector $x$ (i.e., $m \ll n$) such that the lower bound for the relative error is close to~$1$. Thus, the reconstruction function is guaranteed to make a large error for reconstructing at least one $1$-sparse signal.
	
In contrast, Corollary~\ref{cor:approximation_rip} below states that for a $\relu$ network with two hidden layers, recovery of all $s$-sparse vectors is possible to arbitrary precision and in a stable way for not exactly sparse signals or corrupted measurements.

	\begin{corollary} \label{cor:approximation_rip}
		Let $A \in \R^{m \times n}$ be a matrix satisfying the $(2 s, \delta)$-RIP for a $\delta \in (0, 1)$. Then for each $\delta' \in (0, 1)$, there exists a function $\tilde{f}: \R^m \rightarrow \R^n$, represented by an unbiased $\relu$ network with two hidden layers such that for all $x \in \R^n$, $e \in \R^m$,
		\[
		\|\tilde{f}(A x + e) - x\|_2 \leq \delta' \|x\|_2 + C \sigma_s(x)_1 + D \|e\|_2,
		\]
		where $C = 1 +  2\frac{1 + \delta}{1 - \delta}$, $D = \frac{3}{1 - \delta}$, and
		\[
		\sigma_s(x)_1 = \inf_{x' \in \Sigma_s} \|x - x'\|_1.
		\]
	\end{corollary}

\subsection{$\relu$-networks reconstructing signals in more general constraint sets}
	
In this section, we state results beyond the reconstruction of sparse signals from linear measurements. 

We start with Theorem~\ref{thm:lower_bound_generalized}, which implies Corollary~\ref{cor:imposs_one_hidden} stated above but applied to unions of subspaces beyond the set of sparse vectors. 
	
    \begin{theorem} \label{thm:lower_bound_generalized}
		Let $A \in \R^{m \times n}$, $m \leq n$, and $f: \R^m \rightarrow \R^n$ be a function represented by a $\relu$ network with one hidden layer.
		
		Let $x_1, \dots, x_{\tilde{n}} \in \R^n$ be vectors with $\|\cdot\|_2$ norm $1$ and $X := (x_1 \, x_2 \, \dots \, x_{\tilde{n}} ) \in \R^{n \times \tilde{n}}$. Let $U = \bigcup_{k = 1}^{\tilde{n}} \mathrm{span}(x_k)$. Then
		\[
			\sup_{x \in U \backslash \{0\}} \frac{\|f(A x) - x\|_2}{\|x\|_2} \geq \sqrt{\frac{1}{\tilde{n}} \sum_{k = m + 1}^{\tilde{n}} (\sigma_{k}(X)})^2,
		\]
		where $\sigma_k(X)$ is the $k$-th singular value of the matrix $X$.
	\end{theorem}

Next, we state a result that guarantees that ReLU-networks with two layers can implement solvers beyond sparse recovery, applicable to a wider range of inverse problems which, as we show in Section~\ref{sec:inverse_problems}, also includes low-rank matrix recovery and phase retrieval.

	\begin{theorem} \label{thm:approximation_inverse_function}
		Let $U \subset \R^n$ be positive homogeneous. Let $g: \R^n \rightarrow \R^m$ be a positive homogeneous function such that
		\begin{align} \label{eq:inverse_function_prerequisite}
		\inf_{\substack{x^{(1)}, x^{(2)} \in U \\ x^{(1)} \neq x^{(2)}}} \frac{\|g(x^{(1)}) - g(x^{(2)})\|_2}{\|x^{(1)} - x^{(2)}\|_2} =: \tau > 0
		\qquad
		\sup_{\substack{x^{(1)}, x^{(2)} \in \R^n \\ x^{(1)} \neq x^{(2)}}}
		\frac{\|g(x^{(1)}) - g(x^{(2)})\|_2}{\|x^{(1)} - x^{(2)}\|_{II}} =: \rho < \infty
		\end{align}
		where $\|\cdot\|_{II}$ is a norm on $\R^n$ with $\|\cdot\|_2 \leq \|\cdot\|_{II}$.
			Let $\delta \in (0, 1)$. Then there exists a function $f: \R^m \rightarrow \R^n$, represented by a $\relu$ network with two hidden layers, such that for all $x \in \R^n$, $e \in \R^m$,
		\[
		\|f(g(x) + e) - x\|_2 \leq \delta \|x\|_2 + C d_{II}(x, U) + D \|e\|_2,
		\]
		where $C = 1 + \frac{2 \rho}{\tau}$ and $D = \frac{3}{\tau}$ only depend on $\tau$ and $\rho$ and
		\[
		    d_{II}(x, U) := \inf_{x' \in U} \|x - x'\|_{II}.
		\]
	\end{theorem}

	Note that an essential requirement for the lower bounds, Theorem~\ref{thm:lower_bound_generalized} and thus Corollary~\ref{cor:imposs_one_hidden}, to hold is that we consider the approximation errors on the entire (unbounded) domain $U$ or $\Sigma_1$ respectively. The purpose of this is that we are interested in incorporating the positive homogenous structure of the problem into the solution. That is, the recovery function $f$ should by design satisfy $f(\lambda y) = \lambda f(y)$ for all $\lambda \geq 0$ because we know that the measurements $\lambda A x$ are produced by the signal $\lambda x$. If the function $f$ is positive homogeneous and $\|f(A x) - x\|_2 \leq \delta \|x\|_2$ holds for all $x \in \bar{B}_r(0) \cap U$ for some radius $r > 0$, then this is also true on the entire positive homogeneous set $U$. So Theorem~\ref{thm:lower_bound_generalized} and Corollary~\ref{cor:imposs_one_hidden} can also be interpreted in the sense that there is no \textit{positive homogeneous} network function $f$ that provides a good recovery for all vectors in a (possibly bounded) neighborhood of $0$.

	On the other hand, the first part of the proof of Theorem~\ref{thm:lower_bound_generalized} shows that if a positive homogeneous function can be approximated by a $\relu$ network on the entire space, then it is also possible to do this with a positive homogeneous network. Thus, both approaches are even equivalent.
	
	So the key question of this work is when positive homogeneous functions can be approximated (to arbitrary precision) with positive homogeneous networks. 
	In the following subsection, Theorems~\ref{thm:uat_homogeneous} and \ref{thm:uat_negative} provide the answer that it is possible with any number of layers $d \geq 2$ but not with $d = 1$. In addition, Theorem~\ref{thm:uat_homogeneous} also classifies the possible activation functions for this.
	
	\subsection{Universal Approximation of Positive Homogeneous Functions}
	
	We show the following statement about the universal approximation of positive homogeneous functions. It can be seen as an analogous version of the equivalence in Theorem~\ref{thm:universal_approximation} for positive homogeneous functions. The essential proof step for the direction \ref{uat:hom_cond_relu} $\Rightarrow$ \ref{uat_hom_cond_rep} has already been established by Tang et al. \cite{tang2020towards}. We extend this to an equivalence statement, showing that the activation functions described in \ref{uat:hom_cond_relu} are actually the only ones for which the unbiased networks represent a class of functions which are all positive homogeneous and also powerful enough to approximate any other continuous positive homogeneous function.
	
	\begin{theorem} \label{thm:uat_homogeneous}
		Let $\sigma: \R \rightarrow \R$ be a continuous function and $d \geq 2$ an integer. Then the following two statements are equivalent.
		\begin{enumerate}[label=(\alph*)]
			\item \label{uat_hom_cond_rep}
			\begin{itemize}
				\item For every non-empty, closed, positive homogeneous $U \subset \R^n$, every continuous, positive homogeneous function $f: U \rightarrow \R$, and every $\delta > 0$, there exists a function $\tilde{f}: U \rightarrow \R$ that can be represented by a neural network with $d$ hidden layers and activation function $\sigma$, such that for all $x \in U$,
				\[
				|\tilde{f}(x) - f(x)| \leq \delta \|x\|_2
				\] 
				\textbf{and}
				\item every unbiased neural network with $d$ hidden layers and activation function $\sigma$ represents a positive homogeneous function.
			\end{itemize}
			
			\item \label{uat:hom_cond_relu} There are $\alpha, \beta \in \R$, $|\alpha| \neq |\beta|$ such that for all $x \in \R$,
			\[
			\sigma(x) = \alpha \relu(x) + \beta \relu(-x).
			\]
		\end{enumerate}
		
		In case these statements hold, the network representing $\tilde{f}$ in (a) can be chosen to be unbiased.
	\end{theorem}

	In addition, we also complement Theorem~\ref{thm:uat_homogeneous} by the following consequence of Corollary~\ref{cor:imposs_one_hidden} which proves that Theorem~\ref{thm:uat_homogeneous} fails to hold if we consider networks with only one hidden layer. Therefore, its assumption $d \geq 2$ is actually necessary.
	
	\begin{theorem} \label{thm:uat_negative}
	    Let $m \geq 2$, $n \geq 1$ be integers. There exists a continuous, positive homogeneous function $f: \R^m \rightarrow \R^n$ such that for each $\tilde{f}: \R^m \rightarrow \R^n$ that is represented by a $\relu$ network with one hidden layer,
	    \[
	        \sup_{x \in \R^m \backslash \{0\}} \frac{\|\tilde{f}(x) - f(x)\|_2}{\|x\|_2} \geq
	        \begin{cases}
	            \sqrt{1 - \frac{2}{n}} & \text{if } n > 4 \\
	            \sqrt{\frac{n}{8}} & \text{if } n \leq 4.
	        \end{cases}
	    \]
	\end{theorem}
	
	In particular, the case $n = 1$ in Theorem~\ref{thm:uat_negative} shows that contrary to higher depths $d \geq 2$, the first part of \ref{uat_hom_cond_rep} does not hold for the $\relu$ activation function and any $\delta < \sqrt\frac{1}{8}$ if $d = 1$.

    \section{Universal Approximation} \label{sec:universal_approximation}
    
    This section is devoted to the proof of Theorem~\ref{thm:uat_homogeneous}. The following lemmma will be used to establish the implication \ref{uat_hom_cond_rep} $\Rightarrow$ \ref{uat:hom_cond_relu}.
    
    \begin{lemma} \label{lem:relu_activation}
    	Let $\sigma: \R \rightarrow \R$ be a continuous function. Let $k \in \mathbb{Z}_{\geq 1}$. Assume that for all $\gamma \in \R^k$, the function $\sigma_\gamma: \R \rightarrow \R$ defined by
    	\[
    	\sigma_\gamma(x) = \gamma_k \sigma(\gamma_{k - 1} \sigma( \gamma_{k - 2} \sigma( \dots \sigma( \gamma_1 \sigma(x) ) \dots ))),
    	\]
    	with $k$ applications of $\sigma$, is positive homogeneous.
    	
    	Then there exist $\alpha, \beta \in \R$ such that for all $x \in \R$,
    	\begin{align} \label{eq:relu_lincomb}
    		\sigma(x) = \alpha \phi(x) + \beta \phi(-x).
    	\end{align}
    \end{lemma}
    
    \begin{proof}

    	If $k = 1$, then choose $\gamma = 1 \in \R$ such that $\sigma_\gamma = \sigma(x)$. By the positive homogeneity, we obtain that for $x \geq 0$, 
    	\[
    		\sigma(x) = \sigma_\gamma(x \cdot 1) = x \sigma(1) = \sigma(1) \phi(x)
    	\]
    	and for $x \leq 0$,
    	\[
    		\sigma(x) = \sigma_\gamma(|x| \cdot (-1)) = |x| \sigma(-1) = \sigma(-1) \phi(-x),
    	\]
    	showing the representation \eqref{eq:relu_lincomb}.
    	
    	Now we assume that $k \geq 2$. First, we take $\gamma = (1, 0, \dots, 0)^T$. Then for all $x \in \R$, $\sigma_\gamma(x) = 1 \cdot \sigma(0 \cdot \sigma( \dots )) = \sigma(0)$. Since $\sigma_\gamma$ is positive homogeneous, $\sigma(0) = \sigma_\gamma(1) = \frac{1}{2} \sigma_\gamma(2) = \frac{1}{2} \sigma(0)$, so we know that $\sigma(0) = 0$.
    	
    	If $\sigma(x) = 0$ for all $x \geq 0$, then the representation \eqref{eq:relu_lincomb} holds for all $x \geq 0$. Otherwise there exists a $y_0 > 0$ such that $\sigma(y_0) \neq 0$. Then we choose $\gamma_0 = \frac{y_0}{\sigma(y_0)} \in \R$ and $\gamma = (\gamma_0, \dots, \gamma_0)^T \in \R^k$. We define $\tilde{\sigma}: \R \rightarrow \R$ by $\tilde{\sigma}(x) = \gamma_0 \sigma(x)$. Then $\sigma_\gamma = \tilde{\sigma} \circ \dots \circ \tilde{\sigma}$ with $k$ applications.
    	
    	By the choice of $\gamma_0$, $\tilde{\sigma}(y_0) = y_0$ such that also $\sigma_{\gamma}(y_0) = y_0$. By the assumption that $\sigma_{\gamma}$ is positive homogeneous and the previously shown case $k = 1$, we obtain that there are $\eta, \tau \in \R$ such that for all $x \in \R$,
    	\begin{equation} \label{eq:tilde_sigma}
    		\sigma_{\gamma}(x) = \eta \phi(x) + \tau \phi(-x)
    	\end{equation}
    	and $\sigma_{\gamma}(y_0) = y_0$ implies that $\eta = 1$, i.e., $\sigma_{\gamma}(x) = x$ for all $x \geq 0$.
    	
    	So we know that $\sigma_{\gamma}$ is injective on the interval $[0, \infty)$ and then the same must hold for $\tilde{\sigma}$. As a continuous, $\R$-valued, injective function in the interval $[0, \infty)$, $\tilde{\sigma}$ must be either strictly increasing or strictly decreasing. We have already shown $\tilde{\sigma}(0) = 0$ and $\tilde{\sigma}(y_0) = y_0$ where $y_0 > 0$, so it must be strictly increasing. This also implies that $\tilde{\sigma}(x) \geq 0$ for all $x \geq 0$.
    	
    	Now assume that $\tilde{\sigma}(x) = x$ does not hold for all $x \in [0, \infty)$. Then we can find an $x_0 > 0$ such that $\tilde{\sigma}(x_0) \neq x_0$. Starting from this $x_0$, we construct a sequence $(x_j)$ in $[0, \infty)$ by defining $x_{j + 1} = \tilde{\sigma}(x_j)$ for $j = 0, 1, 2, \dots$. We observe the following for all $j \in \mathbb{Z}_{\geq 0}$.
    	\begin{itemize}
    		\item If $x_{j + 1} > x_j$, then by monotonicity also $x_{j + 2} = \tilde{\sigma}(x_{j + 1}) > \tilde{\sigma}(x_j) = x_{j + 1}$.
    		
    		\item If $x_{j + 1} < x_j$, then by monotonicity also $x_{j + 2} = \tilde{\sigma}(x_{j + 1}) < \tilde{\sigma}(x_j) = x_{j + 1}$.
    	\end{itemize}
    	By the choice of $x_0$, we have $x_1 > x_0$ or $x_1 < x_0$ and thus successively either $x_0 < x_1 < x_2 < \dots$ or $x_0 > x_1 > x_2 > \dots$, respectively. In any case, $x_0 \neq x_k$. However, by the definition of the sequence we obtain $x_k = \sigma_{\gamma}(x_0) = x_0$, which is a contradiction. This completes the proof that $\tilde{\sigma}(x) = x$, i.e., $\sigma(x) = \frac{1}{\gamma_0} x$ holds for all $x \geq 0$.
    	
    	To complete the proof also for all $x \leq 0$, we consider the function $\bar{\sigma}: \R \rightarrow \R$, $\bar{\sigma}(x) = - \sigma(- x)$. For this function we obtain that $\bar{\sigma}_{\gamma}(x) = - \sigma_\gamma(- x)$. The latter function is then also positive homogeneous such that by the previous proof there is a $\beta \in \R$ such that for all $x \geq 0$, $\bar{\sigma}(x) = \beta x$ and thus for all $x \leq 0$, $\sigma(x) = - \bar{\sigma}(- x) = - \beta (- x) = \beta x$.		
    \end{proof}
    
    The above statement implies the following corollary.
    
    \begin{corollary} \label{cor:all_pos_hom_relu_representation}
    	Let $\sigma: \R \rightarrow \R$ be a continuous function and $d \in \mathbb{Z}_{\geq 1}$. If all unbiased neural networks with $d$ hidden layers and activation function $\sigma$ represent positive homogeneous functions, then there are $\alpha, \beta \in \R$ such that for all $x \in R$,
    	\[
    	\sigma(x) = \alpha \phi(x) + \beta \phi(- x).
    	\]
    \end{corollary}
    
    \begin{proof}
    	All the functions $\sigma_\gamma$ for $\gamma \in \R^d$ from Lemma \ref{lem:relu_activation} are represented by unbiased networks with $d$ hidden layers with one neuron and weights $\gamma_j$ in each of them and activation function $\sigma$.
    \end{proof}
    
    The next statement will be the core of the proof of \ref{uat:hom_cond_relu} $\Rightarrow$ \ref{uat_hom_cond_rep} and has already been shown in almost the same form in \cite{tang2020towards}.
    
    \begin{theorem} \label{thm:relu_approximation}
    	Let $f: U \rightarrow \R$ be a positive homogeneous, continuous function on a positive homogeneous domain $U \subset \R^m$ and $\epsilon > 0$. Then there exists an unbiased $\relu$ network with two hidden layers, representing the function $\tilde{f}: U \rightarrow \R$, such that for all $x \in U$,
    	\[
    	|\tilde{f}(x) - f(x)| \leq \epsilon \|x\|_2.
    	\]
    \end{theorem}
    
    \begin{proof}
    	The previous work \cite{tang2020towards} shows a similar statement in its supplement in Theorem B.2.2. For completeness of the presentation, we repeat their argument here and adapt it to the situation of Theorem~\ref{thm:relu_approximation}.
    	
    	First, we restrict $f$ to the set $U \cap B_1$ where $B_1 = \{x \in \R^m \,\big|\, \|x\|_1 = 1\}$ is the $\ell_1$ unit ball. Since $U$ is assumed to be closed, this domain $U \cap B_1$ is compact. Therefore, we can apply the universal approximation theorem (Theorem~\ref{thm:universal_approximation}) to obtain a network function $\tilde{g}: U \cap B_1 \rightarrow \R$, $\tilde{g}(x) = W_2 \phi(W_1 x + b_1) + b_2$ where $W_1 \in \R^{k \times m}$, $W_2 \in \R^{1 \times k}$, $b_1 \in \R^{k}$, $b_2 \in \R$ such that for all $x \in U \cap B_1$,
    	\[
    		|\tilde{g}(x) - f(x)| \leq \frac{1}{\sqrt{n}} \epsilon.
    	\]
    	
    	Now we define $\tilde{f}: U \rightarrow \R$ by $\tilde{f}(0) = 0$ and for $x \in U \backslash \{0\}$
    	\[
    	\tilde{f}(x) = \|x\|_1 \tilde{g}(\frac{x}{\|x\|_1})
    	= \|x\|_1 (W_2 \phi(W_1 \frac{x}{\|x\|_1} + b_1) + b_2)
    	= W_2 \phi(W_1 x + \|x\|_1 b_1) + \|x\|_1 b_2.
    	\]
    	Now to show that $\tilde{f}$ can be represented by a $\relu$ network with two hidden layers, we only need to represent the $\|\cdot\|_1$ function with one hidden layer which is done by
    	\[
    	\|x\|_1 = 1_{2 m}^T \relu\left( \begin{pmatrix} Id_m \\ - Id_m \end{pmatrix} x \right),
    	\]
    	where $1_{2 m} \in \R^{2 m}$ is the vector whose all entries are $1$. Substituting this into the above expression for $\tilde{f}$ yields
    	\[
    	\tilde{f}(x)
    	=
    	\begin{pmatrix} W_2 & b_2 \end{pmatrix}
    	\phi \left[
    	\begin{pmatrix}
    		W_1 + b_1 1_m^T & - W_1 + b_1 1_m^T \\
    		1_m^T & 1_m^T
    	\end{pmatrix}
    	\phi \left( \begin{pmatrix} Id_m \\ -Id_m \end{pmatrix} x \right)
    	\right],
    	\]
    	such that $\tilde{f}$ can be represented by an unbiased $\relu$ network with two hidden layers and is therefore also positive homogeneous.
    	
    	Furthermore, for all $x \in U \backslash \{0\}$, $\frac{x}{\|x\|_1} \in U \cap B_1$ such that $|\tilde{g}(\frac{x}{\|x\|_1}) - f(\frac{x}{\|x\|_1})| \leq \frac{\epsilon}{\sqrt{n}}$ and then
    	\begin{align*}
    		|\tilde{f}(x) - f(x)| = \|x\|_1 \left|\tilde{f}(\frac{x}{\|x\|_1}) - f(\frac{x}{\|x\|_1}) \right|
    		= \|x\|_1 \left|\tilde{g}(\frac{x}{\|x\|_1}) - f(\frac{x}{\|x\|_1}) \right| \leq \frac{\epsilon}{\sqrt{n}} \|x\|_1 \leq \epsilon \|x\|_2.
    	\end{align*}

    \end{proof}

    Now we have established all the requirements to prove the specialized universal approximation theorem.
    
    \begin{proof}[Proof of Theorem \ref{thm:uat_homogeneous}]
    	First we show the implication \ref{uat:hom_cond_relu} $\Rightarrow$ \ref{uat_hom_cond_rep}. So let $f: U \rightarrow \R$ be continuous, positive homogeneous and $\delta > 0$. By Theorem \ref{thm:relu_approximation}, there exists an unbiased $\relu$ network with $2$ hidden layers, representing $\tilde{f}: U \rightarrow \R$, such that for all $x \in U$, $|\tilde{f}(x) - f(x)| \leq \delta \|x\|_2$.
    	
    	The one-layer $\relu$ network function $\R^{n_1} \rightarrow \R^{n_1}$
    	\[
    	x \mapsto (Id_{n_1} \quad - Id_{n_1}) \phi \left( \begin{pmatrix} Id_{n_1} \\ -Id_{n_1} \end{pmatrix} x \right) = \phi(x) - \phi(-x) = x,
    	\]
    	is the identity. Thus, we can add $d - 2$ such identity layers to the $2$ layer network representing $\tilde{f}$ without changing the represented function. In this way, $\tilde{f}$ can be represented by an unbiased $\relu$ network with $d$ hidden layers.
    	
    	Since $|\alpha| \neq |\beta|$, we can define $\gamma_1 := \frac{\alpha}{\alpha^2 - \beta^2}$ and $\gamma_2 := \frac{\beta}{\alpha^2 - \beta^2}$ such that
    	\begin{align*}
    		\gamma_1 \sigma(x) - \gamma_2 \sigma(-x) & = 
    		\frac{\alpha}{\alpha^2 - \beta^2} (\alpha \phi(x) + \beta \phi(-x)) - \frac{\beta}{\alpha^2 - \beta^2} (\alpha \phi(-x) + \beta \phi(x)) = \phi(x).
    	\end{align*}
    	Each of the hidden layers in the network for $\tilde{f}$ performs a function $f_j: \R^{n_1} \rightarrow \R^{n_3}$,
    	\[
    	f_j(x) = A \phi(B x)
    	\]
    	for matrices $A \in \R^{n_3 \times n_2}$, $B \in \R^{n_2 \times n_1}$. Then
    	\begin{align*}
    		(\gamma_1 A \quad -\gamma_2 A) \sigma \left( \begin{pmatrix} B \\ -B \end{pmatrix} x \right) & =
    		\gamma_1 A \sigma(B x) - \gamma_2 A \sigma(- B x) = A (\gamma_1 \sigma(B x) - \gamma_2 \sigma(- Bx)) \\
    		& = A \phi(B x) = f_j(x)
    	\end{align*}
    	performs the same operation as one layer with activation function $\sigma$. So we can replace all $\relu$ layers by suitable layers with activation function $\sigma$ and eventually obtain an unbiased network with $d$ hidden layers and activation function $\sigma$ that represents $\tilde{f}$.
    	
    	Now we show the other implication \ref{uat_hom_cond_rep} $\Rightarrow$ \ref{uat:hom_cond_relu}. The second statement of \ref{uat_hom_cond_rep} together with Corollary~\ref{cor:all_pos_hom_relu_representation} implies that there are $\alpha, \beta \in \R$ such that for all $x \in \R$,
    	\[
    	\sigma(x) = \alpha \phi(x) + \beta \phi(-x).
    	\]
    	It remains to show that $|\alpha| \neq |\beta|$. If this is not the case, then $\alpha = \beta$ or $\alpha = - \beta$. In the first case, $\alpha = - \beta$, then $\sigma(x) = \alpha(\phi(x) - \phi(-x)) = \alpha x$ is linear such that also all corresponding network functions $\tilde{f}$ must be affine linear. For example, the function $x \mapsto |x|$ cannot be approximated by such functions to arbitrary precision even though it is positive homogeneous.

    	In the other case, $\sigma(x) = \alpha(\phi(x) + \phi(-x)) = \alpha |x|$.
    	Consider the function $f: \R \rightarrow \R$, $f(x) = \sigma(a x + b)$ for $a, b \in \R$. For this function, it holds that
    	\begin{align} \label{eq:limit_condition}
    		\lim_{x \rightarrow \infty} \frac{f(x)}{|x|}
    		= \lim_{x \rightarrow -\infty} \frac{f(x)}{|x|} \in \R.
    	\end{align}
    	Let $f^{(1)}: \R \rightarrow \R^k$, $f^{(1)}:x \mapsto \sigma(W_1 x + b_1)$ for $W_1 \in \mathbb{R}^{k_1 \times 1}$, $b_1 \in \R^{k_1}$ be the first layer of a network with activation function $\sigma$ (with domain $\R^1$). Then the condition \eqref{eq:limit_condition} holds for each component of $f^{(1)}$.
    	
    	Any linear combination of functions that fulfill \eqref{eq:limit_condition} satisfies \eqref{eq:limit_condition} again. In addition, if $f: \R \rightarrow \R$ satisfies \eqref{eq:limit_condition} with limit $\tau$, then for $\bar{f}: \R \rightarrow \R$, $\bar{f}(x) = |f(x) + c|$, $c \in \R$, it holds that
    	\begin{align*}
    		\lim_{x \rightarrow \pm \infty} \frac{\bar{f}(x)}{|x|} = \lim_{x \rightarrow \pm \infty} \left| \frac{f(x)}{|x|} + \frac{c}{|x|} \right| = |\tau + 0| \in \R,
    	\end{align*}
    	such that \eqref{eq:limit_condition} also holds for $\bar{f}$. In total, we can successively conclude that for any neural network with activation function $\sigma$, any component of any layer, as a function $\R \rightarrow \R$, satisfies \eqref{eq:limit_condition}. Therefore also all functions $f: \R \rightarrow \R$ that are represented by a network with activation function $\sigma$ of any depth and width must satisfy \eqref{eq:limit_condition}. Clearly, $f: \R \rightarrow \R$, $f(x) = x$ is a positive homogeneous function that cannot be approximated by these functions. 
    	
    	So in any case that $|\alpha| = |\beta|$, this contradicts the first part of \ref{uat_hom_cond_rep}.
    \end{proof}

	\section{Inverse Problems} \label{sec:inverse_problems}

This section discusses how well  neural networks with ReLU nonlinearities can solve an inverse problem given by a positive homogeneous dimension-reducing linear forward map. We prove Theorem~\ref{thm:approximation_inverse_function} showing that inverse problems that are well-conditioned in a certain sense can indeed be solved via a ReLU network with two hidden layers before discussing its consequences under the assumption of a restricted isometry property, before proving Theorem~\ref{thm:lower_bound_generalized}, which shows that the analogous results do not hold for ReLU networks with just one hidden layer. The section concludes by discussing the robustness properties.

	\subsection{Proof of the general expressivity result}
	
	Our goal in this subsection is the proof of the general inverse problem Theorem~\ref{thm:approximation_inverse_function}. One ingredient for the proof is Kirszbraun's theorem which is known in functional analysis and measure theory and allows us to extend a Lipschitz continuous function from a subset of $\R^m$ to the entire space.
	
	\begin{theorem}[\textit{Kirszbraun's theorem}, Theorem 2.10.43 in \cite{federer1996}] \label{thm:kirszbraun}
		Let $U \subset \R^n$ and $f: U \rightarrow \R^m$ be a Lipschitz continuous function with Lipschitz constant $L$. Then there exists an extension $g: \R^n \rightarrow \R^m$ of $f$ with Lipschitz constant $L$.
	\end{theorem}
	
	Considering positive homogeneous functions however, Kirszbraun's theorem cannot guarantee that this extension will be positive homogeneous again. However, in the following lemma we show that we can circumvent this by first extending the function to the entire space, then restricting it to the unit sphere and then extend it as a positive homogeneous function again. In this way, the Lipschitz constant will increase by a factor of at most $2$.
	
	\begin{lemma} \label{lem:lipschitz_extension}
		Let $U \subset \R^n$ be non-empty, positive homogeneous and $f: U \rightarrow \R^m$ a function that is positive homogeneous and Lipschitz continuous with constant $L$.
		
		Then there is an extension $\tilde{f}: \R^n \rightarrow \R^m$ of $f$ which is positive homogeneous and Lipschitz continuous with constant $2 L$.
	\end{lemma}
	
    \begin{proof}
    	By the positive homogeneity, $0 \in U$. So we can restrict $f$ to $(S^{n - 1} \cap U) \cup \{0\}$ where it is still Lipschitz continuous with constant $L$. By Theorem \ref{thm:kirszbraun}, we can extend this function to $f: S^{n - 1} \cup \{0\} \rightarrow \R^m$ on the entire set $S^{n - 1} \cup \{0\}$ (by extending to the entire $\R^m$ and then restricting it again) such that it still has Lipschitz constant $L$.
    	
    	Now we define $\tilde{f}: \R^n \rightarrow \R^m$ by $\tilde{f}(x) = \|x\|_2 f(\frac{x}{\|x\|_2})$ for $x \neq 0$ and $\tilde{f}(0) = f(0) = 0$.

    	Clearly, $\tilde{f}$ is positive homogeneous and an extension of $f$ and we will show that it is Lipschitz continuous with constant $2 L$. For this, consider two different points $x, y \in \R^n$. We can assume $\|x\|_2 \leq \|y\|_2$ and thus $y \neq 0$.
    	
    	Then it holds that
    	\[
    		\|x - y\|_2 \geq \left\| x - \frac{\|x\|_2}{\|y\|_2} y \right\|_2
    	\]
    	because
    	\begin{align*}
    		& & \|x - y\|_2^2 & \geq \left\| x - \frac{\|x\|_2}{\|y\|_2} y \right\|_2^2 \\
    		& \Leftrightarrow & \|x\|_2^2 + \|y\|_2^2 - 2 \langle x, y \rangle & \geq \|x\|_2^2 + \|x\|_2^2 - 2 \frac{\|x\|_2}{\|y\|_2} \langle x, y \rangle \\
    		& \Leftrightarrow & \|y\|_2^2 - \|x\|_2^2 & \geq  2 \left( 1 - \frac{\|x\|_2}{\|y\|_2} \right) \langle x, y \rangle \\
    		& \Leftarrow & \|y\|_2^2 - \|x\|_2^2 & \geq  2 \left( 1 - \frac{\|x\|_2}{\|y\|_2} \right) \|x\|_2 \|y\|_2 \\
    		& \Leftrightarrow & \|y\|_2^2 - \|x\|_2^2 & \geq  2 (\|x\|_2 \|y\|_2 - \|x\|_2^2) \\
    		& \Leftrightarrow & \|y\|_2^2 + \|x\|_2^2 - 2(\|x\|_2 \|y\|_2) & \geq 0 \\
    		& \Leftrightarrow & (\|y\|_2 - \|x\|_2)^2 & \geq 0,
    	\end{align*}
    	which is always fulfilled.
    	
    	If also $x \neq 0$, then we can conclude.
    	\begin{align*}
    		\frac{\|\tilde{f}(x) - \tilde{f}(y)\|_2}{\|x - y\|_2} & \leq
    		\frac{\|\tilde{f}(x) - \tilde{f}(\frac{\|x\|_2}{\|y\|_2} y)\|_2}{\|x - y\|_2} + 
    		\frac{\|\tilde{f}(y) - \tilde{f}(\frac{\|x\|_2}{\|y\|_2} y)\|_2}{\|x - y\|_2} \\
    		& \leq
    		\frac{\|\tilde{f}(x) - \tilde{f}(\frac{\|x\|_2}{\|y\|_2} y)\|_2}{\left\| x - \frac{\|x\|_2}{\|y\|_2} y \right\|_2} + 
    		\frac{\left| \|y\|_2 - \|x\|_2 \right| \|\tilde{f}(\frac{y}{\|y\|_2})\|_2 }{\|x - y\|_2} \\
    		& \leq
    		\frac{\|f(\frac{x}{\|x\|_2}) - f(\frac{y}{\|y\|_2})\|_2}{\left\| \frac{x}{\|x\|_2} - \frac{y}{\|y\|_2} \right\|_2} +
    		\frac{\|f(\frac{y}{\|y\|_2}) - f(0)\|_2}{\left\|\frac{y}{\|y\|_2} - 0 \right\|_2}
    		\leq 2 L,
    	\end{align*}
    	and if otherwise $x = 0$,
    	\begin{align*}
    		\frac{\|\tilde{f}(x) - \tilde{f}(y)\|_2}{\|x - y\|_2} & =
    		\frac{\|f(0) - f(\frac{y}{\|y\|_2})\|_2}{\left\|0 - \frac{y}{\|y\|_2} \right\|_2} \leq L,
    	\end{align*}
    	which completes the proof.
    \end{proof}
	
	Now we can use this lemma in the proof for our generalized main theorem for inverse problems.
	
	\begin{proof}[Proof of Theorem \ref{thm:approximation_inverse_function}]
		If there are two $x^{(1)}, x^{(2)} \in U$ with $g(x^{(1)}) = g(x^{(2)})$, then by assumption on $g$, $0 = \|g(x^{(1)}) - g(x^{(2)})\|_2 \geq \tau \|x^{(1)} - x^{(2)}\|_2$. Since $\tau > 0$, this implies that $x^{(1)} = x^{(2)}$. Therefore, $g: U \rightarrow g(U)$ is bijective and has an inverse function $g^{-1} = f_0: g(U) \rightarrow U$.
		
		We obtain
		\begin{align*}
			\sup_{\substack{y^{(1)}, y^{(2)} \in g(U) \\ y^{(1)} \neq y^{(2)}}} \frac{\|f_0(y^{(1)}) - f_0(y^{(2)})\|_2}{\|y^{(1)} - y^{(2)}\|_2} & =
			\sup_{\substack{x^{(1)}, x^{(2)} \in U \\ x^{(1)} \neq x^{(2)}}} \frac{\|x^{(1)} - x^{(2)}\|_2}{\|g(x^{(1)}) - g(x^{(2)})\|_2} \\
			& =
			\left( \inf_{\substack{x^{(1)}, x^{(2)} \in U \\ x^{(1)} \neq x^{(2)}}} \frac{\|g(x^{(1)}) - g(x^{(2)})\|_2}{\|x^{(1)} - x^{(2)}\|_2} \right)^{-1}
			= \frac{1}{\tau}
		\end{align*}
		So $f_0: g(U) \rightarrow U$ is Lipschitz continuous with Lipschitz constant $\frac{1}{\tau}$. By Lemma \ref{lem:lipschitz_extension}, there exists a positive homogeneous extension $f: \R^m \rightarrow \R^n$ with Lipschitz constant $\frac{2}{\tau}$.
		
		For any $x \in \R^n$, $e \in \R^m$ and any $\epsilon > 0$, there exists an $x' \in U$ such that $\|x - x'\|_{II} \leq d_{II}(x, U) + \epsilon$ and then
		\begin{align*}
			\|f(g(x) + e) - x\|_2 & \leq
			\|f(g(x')) - x'\|_2 + \|f(g(x)) - f(g(x'))\|_2 +  \|f(g(x) + e) - f(g(x))\|_2 + \|x - x'\|_2 \\
			& \leq 0 + \frac{2}{\tau} \|g(x) - g(x')\|_2 + \frac{2}{\tau} \|e\|_2 + \|x - x'\|_2 \\
			& \leq \frac{2 \rho}{\tau} \|x - x'\|_{II} + \|x - x'\|_2 + \frac{2}{\tau} \|e\|_2 \\
			& \leq \left(1 + \frac{2 \rho}{\tau} \right) \|x - x'\|_{II} + \frac{2}{\tau} \|e\|_2 \\
			& \leq \left(1 + \frac{2 \rho}{\tau} \right) (d_{II}(x, U) + \epsilon) + \frac{2}{\tau} \|e\|_2,
		\end{align*}
		where we used that $\|\cdot\|_2 \leq \|\cdot\|_{II}$. Since this holds for all $\epsilon > 0$, we must have
		\[
		\|f(g(x) + e) - x\|_2 \leq \left(1 + \frac{2 \rho}{\tau} \right) d_{II}(x, U) + \frac{2}{\tau} \|e\|_2.
		\]
		
		By equivalence of norms, there exists a number $M > 0$ (that possibly depends on the dimension $n$), such that $\|\cdot\|_{II} \leq M \|\cdot\|_2$. $f$ is positive homogeneous and continuous, so by Theorem \ref{thm:uat_homogeneous}, for each component $f_j$ of $f$, there exists an unbiased $\relu$ network with $2$ hidden layers that approximates $f_j$ up to a relative error of $\frac{1}{\sqrt{n}} \min\{\frac{1}{\tau}, \frac{\delta}{\rho M} \} > 0$. Combining these into one network, we obtain an unbiased $\relu$ network with $2$ hidden layers representing $\tilde{f}: \R^m \rightarrow \R^n$ such that for all $y \in \R^m$,
		\[
		\|f(y) - \tilde{f}(y)\|_2 \leq \min \left\{ \frac{1}{\tau}, \frac{\delta}{\rho M} \right\} \|y\|_2.
		\]
		
		Then for all $x \in \R^n$, $e \in \R^m$,
		\begin{align*}
			\|\tilde{f}(g(x) + e) - x\|_2 & \leq
			\|\tilde{f}(g(x) + e) - f(g(x) + e)\|_2 + \|f(g(x) + e) - x\|_2 \\
			& \leq
			\min \left\{ \frac{1}{\tau}, \frac{\delta}{\rho M} \right\} \|g(x) + e\|_2 + \|f(g(x) + e) - x\|_2 \\
			& \leq
			\frac{\delta}{\rho M} \rho \|x\|_{II} + \frac{1}{\tau} \|e\|_2 + \|f(g(x) + e) - x\|_2 \\
			& \leq \delta \|x\|_2 + \left(1 + \frac{2 \rho}{\tau} \right) d_{II}(x, U) + \frac{3}{\tau} \|e\|_2,
		\end{align*}
		where we used $\|x\|_{II} \leq M \|x\|_2$.
	\end{proof}

	\subsection{Restricted Isometries}
	
	The next theorem is an application of the general Theorem~\ref{thm:approximation_inverse_function} for the case of a linear measurement function that satisfies a restricted isometry property on the signal set $U$. This includes the usual restricted isometry property for sparse vectors but also other generalizations like the one for low-rank matrices in \cite{chen2015exact}.
	
	\begin{theorem} \label{thm:generalized_rip}
	    Consider norms $\|\cdot\|_{I}$ on $\R^m$ and $\|\cdot\|_{II}$ on $\R^n$.
	
	    Let $U \subset \R^n$ be a positive homogeneous subset and $A \in \R^{m \times n}$ a linear map such that there are $\delta^{lb} \in (0, 1), \delta^{ub} \in (0, \infty)$ such that for all $x^{(1)}, x^{(2)} \in U$,
	    \begin{align}
	        (1 - \delta^{lb}) \|x^{(1)} - x^{(2)}\|_2 \leq \|A x^{(1)} - A x^{(2)}\|_{I} \leq (1 + \delta^{ub}) \|x^{(1)} - x^{(2)}\|_2.
	        \label{eq:generalized_rip}
	    \end{align}
	    
	    Furthermore, assume that $\|\cdot\|_{I} \leq \alpha \|\cdot\|_2$ on $\R^m$ for an $\alpha \geq 1$, $\|\cdot\|_2 \leq \|\cdot\|_{II}$ on $\R^n$, and each $x \in \R^n$ can be decomposed as $x = x^{(1)} + \dots + x^{(M)}$ where $x^{(1)}, \dots, x^{(M)} \in U$ and $\|x^{(1)}\|_2 + \dots + \|x^{(M)}\|_2 \leq \|x\|_{II}$.
	    
	    Then, for any $\delta' > 0$, there exists an unbiased $\relu$ network with two hidden layers that represents a function $\tilde{f}: \R^m \rightarrow \R^n$ such that for any $x \in \R^n$, $e \in \R^m$,
	    \begin{align*}
	        \|f(A x + e) - x\|_2 \leq \delta' \|x\|_2 + C \alpha d_{II}(x, U) + D \alpha \|e\|_2.
	    \end{align*}
	    where $C = 1 + 2 \frac{1 + \delta^{ub}}{1 - \delta^{lb}}$ and $D = \frac{3}{1 - \delta^{lb}}$ and
	    \begin{align*}
	        d_{II}(x, U) := \inf_{x' \in U} \|x - x'\|_{II}.
	    \end{align*}
	\end{theorem}
	
	\begin{proof}
	    By the assumption on $g$,
	    \begin{align*}
	        \tau := \inf_{\substack{x^{(1)}, x^{(2)} \in U \\ x^{(1)} \neq x^{(2)} }} \frac{\|A x^{(1)} - A x^{(2)} \|_2}{\|x^{(1)} - x^{(2)}\|_2} \geq 
	        \frac{1}{\alpha} \inf_{\substack{x^{(1)}, x^{(2)} \in U \\ x^{(1)} \neq x^{(2)} }} \frac{\| A x^{(1)} - A x^{(2)}\|_I}{\|x^{(1)} - x^{(2)}\|_2}
	        \geq \frac{1}{\alpha} (1 - \delta^{lb}) > 0,
	    \end{align*}
	    such that the first assumption of Theorem~\ref{thm:approximation_inverse_function} is fulfilled.
	    
	    For any $x^{(1)}, x^{(2)} \in \R^n$, define $z := x^{(1)} - x^{(2)}$. By assumption, there is a decomposition $z = z^{(1)} + \dots + z^{(M)}$ such that $z^{(1)}, \dots, z^{(M)} \in U$ and $\|z^{(1)}\|_2 + \dots + \|z^{(M)}\|_2 \leq \|z\|_{II}$. Then
	    \begin{align*}
	        \|A x^{(1)} - A x^{(2)}\|_2 = \|A z\|_2 \leq \sum_{j = 1}^M \|A z^{(j)}\|_2 \leq (1 + \delta^{ub}) \sum_{j = 1}^M \|z^{(j)}\|_2 \leq (1 + \delta^{ub}) \|x^{(1)} - x^{(2)}\|_{II},
	    \end{align*}
	    and therefore
	    \begin{align*}
	        \rho := \sup_{\substack{x^{(1)}, x^{(2)} \in \R^n \\ x^{(1)} \neq x^{(2)}} } \frac{\|A x^{(1)} - A x^{(2)}\|_2}{\|x^{(1)} - x^{(2)}\|_{II}} \leq 1 + \delta^{ub}.
	    \end{align*}
	    
	    So we can apply Theorem~\ref{thm:approximation_inverse_function} and obtain that for any $\delta' > 0$, there exists a function $\tilde{f}: \R^m \rightarrow \R^n$ such that for any $x \in \R^n$ and $e \in \R^m$,
	    \begin{align*}
	        \|f(A x + e) - x\|_2 & \leq \delta' \|x\|_2 + \left(1 + \frac{2 \rho}{\tau} \right) d_{II}(x, U) + \frac{3}{\tau} \|e\|_2 \\
	        & \leq \delta' \|x\|_2 + C \alpha d_{II}(x, U) + D \alpha \|e\|_2 
	    \end{align*}
	    where $C = 1 + 2 \frac{1 + \delta^{ub}}{1 - \delta^{lb}}$ and $D = \frac{3}{1 - \delta^{lb}}$.
	    
	\end{proof}
	
	A first immediate consequence from the above theorem is the main result about sparse recovery for matrices with the restricted isometry property.
	
	\begin{proof}[Proof of Corollary~\ref{cor:approximation_rip}]
	    If $A$ satisfies the $(s, \delta)$-restricted isometry property (for sparse vectors), then \eqref{eq:generalized_rip} is fulfilled for $\delta^{lb} = \delta^{ub} = \delta$, $U = \Sigma_s$ and $\|\cdot\|_I = \|\cdot\|_2$. Furthermore, we choose $\|\cdot\|_{II} = \|\cdot\|_1$ such that any $x \in \R^n$ can be decomposed as $x = \sum_{j = 1}^n x_j e_j$ where the $e_j \in \R^n$ are the canonical basis vectors. Then clearly, each $x_j e_j \in U$ and $\sum_{j = 1}^n \|x_j e_j\|_2 = \sum_{j = 1}^n |x_j| = \|x\|_1 = \|x\|_{II}$. Then Theorem~\ref{thm:generalized_rip} implies Corollary~\ref{cor:approximation_rip}.
	\end{proof}
	
	Another application of Theorem~\ref{thm:generalized_rip} is low-rank matrix recovery. Besides sparse vectors, the inequality \eqref{eq:generalized_rip} has also been studied for linear operators on low-rank matrices. Using these results, we can prove the following consequence of Theorem~\ref{thm:generalized_rip}.
	
	\begin{corollary} \label{cor:low_rank_recovery}
	    There are universal constants $C, D, C_3, c_3 > 0$ such that the following holds.
	    
	    Let $A \in \R^{m \times n}$ have i.i.d.~subgaussian entries $A_{j, k}$ satisfying
	    \begin{align*}
	        \mathbb{E}[A_{j, k}] & = 0 & \mathbb{E}[A_{j, k}^2] & = 1 & \mathbb{E}[A_{j, k}^4] > 1.
	    \end{align*}
	    Define the operator $\mathcal{A}: \R^{n \times n} \rightarrow \R^m$ such that for all $X \in \R^{n \times n}$,
	    \begin{align*}
	        (\mathcal{A}(X))_j = \sum_{k, l = 1}^n A_{j, k} A_{j, l} X_{k, l}.
	    \end{align*}
	    
	    Let $1 \leq r \leq n$ be an integer and $m \geq c_4 n r$. Then with probability $\geq 1 - C_3 e^{-c_3 m}$, the following holds: For any $\delta' > 0$, there exists function $\tilde{f}: \R^m \rightarrow \R^{n \times n}$, represented by an unbiased $\relu$ network with $2$ hidden layers, such that for all $X \in \R^{n \times n}$ and $e \in \R^m$,
	    \begin{align*}
	        \|\tilde{f}(\mathcal{A}(X) + e) - X\|_F \leq
	        \delta'\|X\|_F + C \sqrt{m} d_*(X, U_r) + \frac{D}{\sqrt{m}} \|e\|_2,
	    \end{align*}
	    where $U_r \subset \R^{n \times n}$ is the set of rank $\leq r$ matrices and $d_*$ denotes the distance in $\|\cdot\|_*$ (nuclear norm).
	\end{corollary}
	
	\begin{proof}
	    According to Corollary 1 in \cite{chen2015exact}, with probability $\geq 1 - C_3 e^{-c_3 m}$, $\mathcal{A}$ satisfies the RIP for low-rank matrices in the sense that for all $X \in \R^{n \times n}$ of rank $\leq 2 r$,
	    \begin{align*}
	        (1 - \delta^{lb}) \|X\|_F \leq \frac{1}{m} \|\mathcal{A}(X)\|_1 \leq (1 + \delta^{ub}) \|X\|_F
	    \end{align*}
	    for universal constants $\delta^{lb} \in (0, 1)$ and $\delta^{ub} > 0$. Therefore, for any $X^{(1)}, X^{(2)} \in \R^{n \times n}$ of rank $\leq r$, 
	    \begin{align*}
	        (1 - \delta^{lb}) \|X^{(1)} - X^{(2)}\|_F \leq \frac{1}{m} \|\mathcal{A}(X^{(1)} - X^{(2)})\|_1 \leq (1 + \delta^{ub}) \|X^{(1)} - X^{(2)}\|_F,
	    \end{align*}
	    such that \eqref{eq:generalized_rip} is fulfilled for the $\|\cdot\|_F$ norm which corresponds to the $\|\cdot\|_2$ norm of the vectorized matrices and $\|\cdot\|_I = \|\cdot\|_1$ on $\R^m$. Then $\|\cdot\|_I \leq \alpha \|\cdot\|_2$ for $\alpha = \sqrt{m}$. 
	    Furthermore, define $U \subset \R^{n \times n}$ to be the set of rank $\leq r$ matrices. Then $U$ is positive homogeneous. Define $\|\cdot\|_{II}$ to be the nuclear norm $\|\cdot\|_*$. Then any matrix $X \in \R^{n \times n}$ has a singular value decomposition $\sum_{j = 1}^{n} \sigma_{j} u_j v_j^*$ with singular values $\sigma_1, \dots, \sigma_n$ and orthonormal $u_1, \dots, u_n$ and $v_1, \dots, v_n$ in $\R^n$. Then every $\sigma_j u_j v_j^*$ is in $U$ and
	    \begin{align*}
	        \sum_{j = 1}^n \|\sigma_j u_j v_j^*\|_F = \sum_{j = 1}^n \sigma_j = \|X\|_* = \|X\|_{II}.
	    \end{align*}
	    
	    Then by Theorem~\ref{thm:generalized_rip}, for each $\delta' > 0$, there exists a $\hat{\tilde{f}}: \R^m \rightarrow \R^n$, represented by a $\relu$ network with two hidden layers, such that for all $X \in \R^{n \times n}$ and $e \in \R^m$,
	    \begin{align*}
	        \|\hat{\tilde{f}}(\frac{1}{m}\mathcal{A}(X) + \frac{e}{m}) - X\|_F \leq \delta' \|X\|_F + C \sqrt{m} d_*(X, U_r) + D \sqrt{m} \|\frac{e}{m}\|_2
	    \end{align*}
	    and thus if we define $\tilde{f}(y) = \hat{\tilde{f}}(\frac{1}{m} y)$, which can also be represented by a $\relu$ network with two hidden layers,
	    \begin{align*}
	        \|\tilde{f}(\mathcal{A}(X) + e) - X\|_F \leq \delta' \|X\|_F + C \sqrt{m} d_*(X, U_r) + \frac{D}{\sqrt{m}} \|e\|_2
	    \end{align*}
	    for $C = 1 + 2 \frac{1 + \delta^{ub}}{1 - \delta^{lb}}$ and $D = \frac{3}{1 - \delta^{lb}}$.
	\end{proof}
	
	\begin{remark}
	    Corollary~\ref{cor:low_rank_recovery} has the error dependence $\frac{1}{\sqrt{m}} \|e\|_2$. This is worse or equal to the dependence $\frac{1}{m} \|e\|_1$ in Theorem 1 of \cite{chen2015exact}. This is caused by the upper bound $\|\cdot\|_1 \leq \sqrt{m} \|\cdot\|_2$ which is needed because we apply Kirszbraun's theorem for the $\ell_2$ norm. This could be improved if Kirszbraun's theorem also holds for functions on a domain with the $\ell_1$ norm. The same holds for the additional $\sqrt{m}$ factor in the dependence on $d_*(X, U_r)$.
	\end{remark}
	
    \begin{remark}
    	\begin{itemize}
    	\item
    	In Corollary~\ref{cor:low_rank_recovery}, if $r = 1$, the operator $\mathcal{A}$ applied to rank $1$ matrices of the type $x x^*$ for $x \in \R^n$, yields
    	\begin{align*}
    		(\mathcal{A}(x x^*))_j = \sum_{k, l = 1}^n A_{j, k} A_{j, l} x_k x_l = \left|\sum_{k = 1}^n A_{j, k} x_k \right|^2 = |(A x)_j|^2,
    	\end{align*}
    	so $\mathcal{A}(x x^*) = |A x|^2$, where $|A x|^2$ contains the squared absolute values of the entries of $A x$. So the $\relu$ network function $\tilde{f}$ can (approximately) reconstruct $x x^*$ (and therefore indirectly $x$) from $|A x|^2$, which is the widely studied phase retrieval problem.
    
    	\item
    	So Corollary~\ref{cor:low_rank_recovery} enables us to solve the phase retrieval problem in the sense that from $|A x|^2$ we can calculate $x x^*$ using an end-to-end network. One might wonder whether it is also possible to calculate the vector $x$ from $|A x|^2$ or $|A x|$ directly. However, the problem is that this $x$ is not unique since for any $|\lambda| = 1$ (thus $\lambda = \pm 1$ in $\R$), $|A x| = |A \lambda x|$. It is not even possible to define a continuous function $f: \R^m \rightarrow \R^n$ such that for each $y$, $|A f(y)| = y$ if $A$ enables a unique solution of the phase retrieval problem up to global phase. To see this, define $g(x) = |A x|$ which is continuous. If $f$ is continuous, then also $f \circ g$. Let $e_1$ be the first canonical basis vector. If the phase retrieval problem for $A$ is uniquely solvable, then we need to have $f(g(e_1)) = \pm e_1$. Without loss of generality, we can assume $f(g(e_1)) = + e_1$. Now consider $f \circ g$ along the connected line $t \mapsto (1 - t)e_1 + t e_2$ ($t \in [0, 1]$). By the uniqueness of the solution up to global phase, we must have $(f \circ g)((1 - t) e_1 + t e_2) = \pm ((1 - t) e_1 + t e_2)$ for each $t \in [0, 1]$. Since we assume that the $\pm 1$ sign is $+1$ for $t = 0$ and always $(1 - t) e_1 + t e_2 \neq 0$, by continuity we know that $(f \circ g)((1 - t) e_1 + t e_2) = + ((1 - t) e_1 + t e_2)$ for all $t \in [0, 1]$. Especially, $(f \circ g)(e_2) = e_2$. Analogously, new we can consider $f \circ g$ along the connected line from $e_2$ to $-e_1$ and conclude that $(f \circ g)(-e_1) = -e_1$. However, $g(-e_1) = |A(-e_1)| = g(e_1)$ and therefore also $-e_1 = (f \circ g)(-e_1) = (f \circ g)(e_1)$. This contradicts the previous assumption that $(f \circ g)(e_1) = + e_1$.
    	\end{itemize}
    \end{remark}

	\subsection{Lower Bounds}
	
	In this subsection, we prove the lower bounds for sparse recovery (Corollary~\ref{cor:imposs_one_hidden}, Theorem~\ref{thm:lower_bound_generalized}) and subsequently also for universal approximation (Theorem~\ref{thm:uat_negative}).

	\begin{proof}[Proof of Theorem \ref{thm:lower_bound_generalized}]
	
	    Let $f = W_2(W_1 x + b_1) + b_2$ be the network function with $W_1 \in \R^{k \times m}$, $W_2 \in \R^{m \times k}$, $b_1 \in \R^k$, $b_2 \in \R^n$.
		
		We first show that it is sufficient to prove the statement for networks with zero biases ($b_1=0$ and $b_2=0$), because if we scale the signal $x$ with a sufficiently large constant, the biases become irrelevant. 
		
		To see more formally that that we can set $b_1=0$ and $b_2=0$, recall that we denote $\phi$ for the $\relu$ function. Note that for any numbers $\lambda, a \in \mathbb{R}$ and $\lambda \geq 0$, $\phi(\lambda a) = \lambda \phi(a)$.
		Note that by this observation and the continuity of $\phi$ and $\|\cdot\|_2$,
		\begin{align*}
			\sup_{x \in U \backslash \{0\}} \frac{\|x - f(A x)\|_2}{\|x\|_2} 
			&= 
			\sup_{x \in U \backslash \{0\}} \sup_{\lambda > 0} \frac{\|\lambda x - f(A \lambda x)\|_2}{\|\lambda x\|_2} 
			= 
			\sup_{x \in U \backslash \{0\}} \sup_{\lambda > 0} \frac{\|\lambda x - W_2 \phi(W_1 A \lambda x + b_1) - b_2 \|_2}{\|\lambda x\|_2} \\
			&= 
			\sup_{x \in U \backslash \{0\}} \sup_{\lambda > 0} \frac{\|x - W_2 \phi(W_1 A x + \frac{b_1}{\lambda}) - \frac{b_2}{\lambda} \|_2}{\|x\|_2} \\
			&\geq 
			\sup_{x \in U \backslash \{0\}} \lim_{\lambda \rightarrow \infty} \frac{\|x - W_2 \phi(W_1 A x + \frac{b_1}{\lambda}) - \frac{b_2}{\lambda} \|_2}{\|x\|_2} 
			= 
			\sup_{x \in U \backslash \{0\}} \frac{\|x - W_2 \phi(W_1 A x) \|_2}{\|x\|_2}.
		\end{align*}
		So it is sufficient to prove the statement for a network $f$ with no biases (i.e., with $b_1, b_2 = 0$). 
		
		$f$ is defined on $\R^m$ and for a matrix $M \in \R^{m \times m'}$ we define $f(M) \in \R^{n \times m'}$ to be the column-wise application of $f$ on $M$. Then for the matrix $X \in \R^{n \times \tilde{n}}$ whose columns are the vectors $x_1, \dots, x_{\tilde{n}}$, we obtain
		\begin{align*}
		    \|f(A X) - X\|_F^2 = \sum_{k = 1}^{\tilde{n}} \|f(A x_k) - x_k\|_2^2
		\end{align*}
		as the sum of the squared deviations. Now we find a lower bound for $\|f(A X) - X\|_F$.
		
		We observe that
		\begin{align*}
		    f(A X) - f(- A X) & = W_2 \phi(W_1 A X) - W_2 \phi(-W_1 A X)
		    = W_2 \left[ \phi(W_1 A X) - \phi(- W_1 A X) \right].
		\end{align*}
		For any $x \in \R$, $\phi(x) - \phi(-x) = x$ such that
		\begin{align*}
		    f(A X) - f(-A X) = W_2 W_1 A X.
		\end{align*}
		Since $A \in \R^{m \times n}$, it has rank $\leq m$ and therefore also $f(A X) - f(-A X)$ has rank $\leq m$. Our goal is to bound $\|f(A X) - X\|_F$ or $\|f(-A X) - (-X)\|_F$ for which we bound $\|f(A X) - f(-A X) - 2 X\|_F$. By the Eckart-Young-Mirsky theorem, the best rank $m$ approximation of $2 X$ in Frobenius norm can be obtained by truncating its sigular value decomposition after the largest $m$ singular values and therefore, for any rank $\leq m$ matrix $M \in \R^{n \times \tilde{n}}$,
		\begin{align*}
		    \|M - 2 X\|_F^2 \geq \sum_{k = m + 1}^{\tilde{n}} (\sigma_k(2 X))^2.
		\end{align*}
		Note that this even holds if $X$ itself has rank $\leq m$, in which case $\sigma_k(X) = 0$ for all $m + 1 \leq k \leq \tilde{n}$. Since $f(A X) - f(- A x)$ has rank $\leq m$,
		\begin{align*}
		    2 \alpha & := 2 \sqrt{\sum_{k = m + 1}^{\tilde{n}} (\sigma_k(X))^2} \leq \|f(A X) - f(- A X) - 2 X\|_F \\
		    & \leq \|f(A X) - X\|_F + \|f(- A X) - (-X) \|_F \leq 2 \max\left\{\|f(A X) - X\|_F, \|f(- A X) - (-X) \|_F \right\}.
		\end{align*}
		So one of these norms on the right hand side is $\geq \alpha$. W.l.o.g.~we assume that it is the first one.
		Then
		\begin{align*}
		    \sum_{k = m + 1}^{\tilde{n}} (\sigma_k(X))^2 \leq \|f(A X) - X\|_F^2
		    = \sum_{k = 1}^{\tilde{n}} \|f(A x_k) - x_k\|_2^2
		    \leq \tilde{n} \max_{k \in [\tilde{n}]} \|f(A x_k) - x_k\|_2^2.
		\end{align*}
		So we can conclude
		\begin{align*}
	        \sup_{x \in U \backslash \{0\}} \frac{\|f(A x) - x\|_2}{\|x\|_2}
	        \geq
		    \max_{k \in [\tilde{n}]} \|f(A x_k) - x_k\|_2
		    \geq 
		    \sqrt{\frac{1}{\tilde{n}} \sum_{k = m + 1}^{\tilde{n}} (\sigma_k(X))^2 }.
		\end{align*}
	\end{proof}

	\begin{proof}[Proof of Corollary \ref{cor:imposs_one_hidden}]
	    Corollary~\ref{cor:imposs_one_hidden} follows from Theorem~\ref{thm:lower_bound_generalized} by choosing $x_1 = e_1, \dots, x_n = e_n$. Then $X = (x_1 \dots x_n) = Id_n$ and $U = \Sigma_1$. The lower bound simplifies to
	    \begin{align*}
	        \sqrt{\frac{1}{n} \sum_{k = m + 1}^{n} (\sigma_k(Id_n))^2} = \sqrt{\frac{1}{n}(n - m)} = \sqrt{1 - \frac{m}{n}}.
	    \end{align*}
	\end{proof}
	
	Using the lower bound for the specific case of sparse recovery, we can also show a lower bound for the general approximation of continuous, positive homogeneous functions.
	
	\begin{proof}[Proof of Theorem~\ref{thm:uat_negative}]
	    Let $w_1, \dots, w_n \in \R$ be pairwise distinct. Define the matrix
	    \begin{align*}
	        A :=
	        \begin{pmatrix}
	                \frac{1}{\sqrt{1 + w_1^2}} & \dots & \frac{1}{\sqrt{1 + w_n^2}} \\
	                \frac{w_1}{\sqrt{1 + w_1^2}} & \dots & \frac{w_n}{\sqrt{1 + w_n^2}}
	        \end{pmatrix}
	        \in \R^{2 \times n}.
	    \end{align*}
	    
	    All the columns of $A$ have an $\ell_2$-norm of $1$. Furthermore, if there exists a $2$-sparse $x \in \R^n$ such that $A x = 0$, then there is a $2 \times 2$-subdeterminant which is $0$, i.e., for some $k, l \in [n]$,
	    \begin{align*}
	        0 = \det \begin{pmatrix}
	                \frac{1}{\sqrt{1 + w_k^2}} & \frac{1}{\sqrt{1 + w_l^2}} \\
	                \frac{w_k}{\sqrt{1 + w_k^2}} & \frac{w_l}{\sqrt{1 + w_l^2}}
	        \end{pmatrix}
	        = \frac{w_l - w_k}{\sqrt{1 + w_k^2} \sqrt{1 + w_l^2}}
	    \end{align*}
	    and thus $w_l = w_k$, contradicting the assumption that the numbers are pairwise distinct.
	    
	    So $A x \neq 0$ must hold for all $2$-sparse $x \in \R^n$. Furthermore, $A$ is injective on $\Sigma_1$. So there exists an inverse map $f: A\Sigma_1 \rightarrow \Sigma_1$. The set $\Sigma_1 \cap S^{n - 1}$ is compact and thus there exists
	    \begin{align*}
	       \tau := \min_{x \in \Sigma_2 \cap S^{n - 1}} \|A x\|_2 = \min_{x \in \Sigma_2 \backslash \{0\}} \frac{\|A x\|_2}{\|x\|_2} = \min_{\substack{x, y \in \Sigma_1 \\ x \neq y}} \frac{\|A x - A y\|_2}{\|x - y\|_2}
	    \end{align*}
	    and by the previous observation, $\tau > 0$. So for all $x, y \in \Sigma_1$, $x \neq y$, $\|x - y\|_2 \leq \frac{1}{\tau} \|A x - A y\|_2$ such that the inverse map $f$ is Lipschitz continuous with Lipschitz constant $\frac{1}{\tau}$.
	    
	    Furthermore, since $A$ as a function is positive homogeneous, also its restricted inversion $f$ must be positive homogeneous. By Lemma~\ref{lem:lipschitz_extension}, there exists a positive homogeneous extension $f: \R^2 \rightarrow \R^n$ on the entire space with Lipschitz constant $\frac{2}{\tau}$.
	    
	    Now let $\tilde{f}: \R^2 \rightarrow \R^n$ be any function that can be represented by a $\relu$ network with one hidden layer. By Corollary~\ref{cor:imposs_one_hidden},
	    \begin{align*}
	        \sup_{x \in \Sigma_1 \backslash \{0\}} \frac{\|x - \tilde{f}(A x)\|_2}{\|x\|_2} \geq \sqrt{1 - \frac{2}{n}}.
	    \end{align*}
	    Now for each $x \in \Sigma_1$, by the definition of $f$, $f(A x) = x$ and since $A$ has normalized columns, $\|A x\|_2 = \|x\|_2$. Therefore we can conclude
	    \begin{align*}
	        \sqrt{1 - \frac{2}{n}} & \leq \sup_{x \in \Sigma_1 \backslash \{0\}} \frac{\|x - \tilde{f}(A x)\|_2}{\|x\|_2} = \sup_{x \in \Sigma_1 \backslash \{0\}} \frac{\|f(A x) - \tilde{f}(A x)\|_2}{\|A x\|_2} \\
	        & = \sup_{y \in A \Sigma_1 \backslash \{0\}} \frac{\|f(y) - \tilde{f}(y)\|_2}{\|y\|_2} \leq \sup_{y \in \R^2 \backslash \{0\}} \frac{\|\tilde{f}(y) - f(y)\|_2}{\|y\|_2}.
	    \end{align*}
	    $f$ is positive homogeneous and Lipschitz continuous, thus also continuous. We can expand $f: \R^2 \rightarrow \R^n$ to a function $\bar{f}: \R^m \rightarrow \R^n$ by setting $\bar{f}(y_1, \dots, y_m) = f(y_1, y_2)$. In this way, $\bar{f}$ is still positive homogeneous and continuous. For any $\tilde{\bar{f}}: \R^m \rightarrow \R^n$ that is represented by a $\relu$ network with one hidden layers, also $\tilde{f}: \R^2 \rightarrow \R^n$, $\tilde{f}(y) = \tilde{\bar{f}}(y_1, y_2, 0, \dots, 0)$ can be represented by a $\relu$ network with one hidden layer such that
	    \begin{align}
	        \sup_{\tilde{y} \in \R^m \backslash \{0\}} \frac{\|\tilde{\bar{f}}(\tilde{y}) - \bar{f}(\tilde{y})\|_2}{\|\tilde{y}\|_2}
	        & \geq
	        \sup_{y \in \R^2 \backslash \{0\}} \frac{\|\tilde{\bar{f}}(y_1, y_2, 0, \dots, 0) - \bar{f}(y_1, y_2, 0, \dots, 0)\|_2}{\|y\|_2} \nonumber \\
	        & =
	        \sup_{y \in \R^2 \backslash \{0\}} \frac{\|\tilde{f}(y) - f(y)\|_2}{\|y\|_2}
	        \geq \sqrt{1 - \frac{2}{n}}.
	        \label{eq:f_bar_lower_bound}
	    \end{align}

	    Now take an $n' \leq n$ and let $\bar{f}: \R^m \rightarrow \R^n$ as above such that \eqref{eq:f_bar_lower_bound} holds for all $\relu$ networks $\tilde{\bar{f}}$ with one hidden layer. For each subset $S \subset [n]$ of size $|S| = n'$, we can define a function $\bar{f}_S: \R^{m} \rightarrow \R^{n}$ such that for all $y \in \R^m$, $j \in [n]$,
	    \begin{align*}
	        (\bar{f}_S(y))_j := \begin{cases}
	            (\bar{f}(y))_j & \text{if } j \in S \\
	            0 & \text{otherwise.}
	        \end{cases}
	    \end{align*}
	    Now assume that every continuous, positive homogeneous function $\R^{m} \rightarrow \R^{n'}$ can be approximated by a one-layer $\relu$ network up to relative precision $n'(\frac{1}{n} - \frac{2}{n^2}) - \epsilon$ for an $\epsilon > 0$. Then for each $S \subset [n]$, $|S| = n'$, there exists a $\relu$ network function $\tilde{\bar{f}}_{(S)}: \R^m \rightarrow \R^n$ such that $(\tilde{\bar{f}}_{(S)})_j = 0$ for $j \in [n] \backslash S$ and for all $y \in \R^m$,
	    \begin{align*}
	        \frac{\|\tilde{\bar{f}}_{(S)}(y) - \bar{f}_S(y)\|_2^2}{\|y\|_2^2} \leq n' \left( \frac{1}{n} - \frac{2}{n^2} \right) - \epsilon.
	    \end{align*}
	    Since every $j \in [n]$ is contained in exactly $\binom{n - 1}{n' - 1}$ subsets $S \subset [n]$, for each $y \in \R^m$, $j \in [n]$,
	    \begin{align*}
	        \bar{f}_j(y) = \frac{1}{\binom{n - 1}{n' - 1}} \sum_{\substack{S \subset [n] \\ \text{s.t. } |S| = n'}} (\bar{f}_S(y))_j.
	    \end{align*}
	    We define the $\relu$ network function $\tilde{\bar{f}}: \R^m \rightarrow \R^n$ with one hidden layer by
	    \begin{align*}
	        \tilde{\bar{f}}(y) = \frac{1}{\binom{n - 1}{n' - 1}} \sum_{\substack{S \subset [n] \\ \text{s.t. } |S| = n'}} \tilde{\bar{f}}_{(S)}(y).
	    \end{align*}
	    Then by \eqref{eq:f_bar_lower_bound}, we have
	    \begin{align*}
	        \sup_{y \in \R^m \backslash \{0\}} \frac{\|\tilde{\bar{f}}(y) - \bar{f}(y)\|_2^2}{\|y\|_2^2}
	        \geq 1 - \frac{2}{n}
	    \end{align*}
	    
	    On the other hand, for all $y \in \R^m \backslash \{0\}$,
	    \begin{align*}
	        \frac{\|\tilde{\bar{f}}(y) - \bar{f}(y)\|_2}{\|y\|_2}
	        & = \sum_{j \in [n]} \frac{|\tilde{\bar{f}}_j(y) - \bar{f}_j(y)|^2}{\|y\|_2}
	        \leq \frac{1}{\binom{n - 1}{n' - 1}} \sum_{j \in [n]} \sum_{\substack{S \subset [n] \\ \text{s.t. } |S| = n'}} \frac{|(\tilde{\bar{f}}_{(S)}(y))_j - (\bar{f}_S (y))_j|^2}{\|y\|_2} \\
	        & \leq \frac{1}{\binom{n - 1}{n' - 1}} \sum_{\substack{S \subset [n] \\ \text{s.t. } |S| = n'}}  \sum_{j \in S} \frac{|(\tilde{\bar{f}}_{(S)}(y))_j - (\bar{f}_S (y))_j|^2}{\|y\|_2}
	        = \frac{1}{\binom{n - 1}{n' - 1}} \sum_{\substack{S \subset [n] \\ \text{s.t. } |S| = n'}}  \frac{\|\tilde{\bar{f}}_{(S)}(y) - \bar{f}_S(y)\|_2^2}{\|y\|_2} \\
	        & \leq \frac{\binom{n}{n'}}{\binom{n - 1}{n' - 1}} \left[ n' \left( \frac{1}{n} - \frac{2}{n^2} \right) - \epsilon \right]
	        = \frac{n}{n'} \left[ n' \left( \frac{1}{n} - \frac{2}{n^2} \right) - \epsilon \right]
	        = 1 - \frac{2}{n} - \epsilon \frac{n}{n'} < 1 - \frac{2}{n}.
	    \end{align*}
	    This is a contradiction. So we can concluded that for $n' \leq n$, there exists a function $f: \R^m \rightarrow \R^{n'}$ such that for all one-layer $\relu$ networks $\tilde{f}: \R^m \rightarrow \R^{n'}$,
	    \begin{align*}
	        \sup_{y \in \R^m \backslash \{0\}} \frac{\|\tilde{f}(y) - f(y)\|_2^2}{\|y\|_2^2} \geq n' \left( \frac{1}{n} - \frac{2}{n^2} \right).
	    \end{align*}
	    The second factor $\left( \frac{1}{n} - \frac{2}{n^2} \right)$ becomes maximal for $n = 4$. For $n' \leq 4$, we choose $n = 4$ and otherwise $n = n'$, which proves the bound from the theorem statement.

	\end{proof}
 
	 \subsection{Robustness\label{sec:nn_robust}}
	 Our main theorems also address the robustness of our solution already mentioned at the beginning in Section~\ref{subsec:nn_robustness_intro}. We can obtain similar guarantees to minimization-based approaches. This agrees with the empirical observation in \cite{genzel2022solving} that states that neural networks provide a similar robustness to total variation minimization which is related to $\ell_1$ minimization for sparse recovery. Note however, that in our work, we only study the existence of the networks but not training them.

	The robustness seemingly contradicts the analysis of \cite{gottschling2020troublesome} which analyzes certain scenarios in which problems with the robustness of neural network for inverse problems occur. In particular, they show (Theorem 3.1 in \cite{gottschling2020troublesome}) that instabilities have to occur if one tries to recover signals whose difference is close to the kernel of the measurement matrix, i.e., $x$ and $x'$ such that $\|x - x'\|$ is large compared to $\|A x - A x'\|$. Avoiding this situation is referred to as \textit{kernel awareness} (discussed in Section~4.2 in \cite{gottschling2020troublesome}) which for sparse recovery can be achieved if $\|x\|_2 \leq \gamma \|A x\|_2$ for a constant $\gamma > 0$ and all $2 s$-sparse vectors. In fact, \eqref{eq:inverse_function_prerequisite} in Theorem~\ref{thm:approximation_inverse_function} is exactly such a kernel awareness condition which ensures that the considered problems are well-behaved in this respect. This condition is ensured by the $\ell_2$-robust null space property and therefore also the restricted isometry property which are assumed in many compressed sensing scenarios.
	
	However, by a theoretical comparison of neural networks to an optimization approach similar to \eqref{eq:qc_bp_nn_introduction} (Theorem~6.3 in \cite{gottschling2020troublesome}), they show that even for a problem related to sparse recovery, neural networks necessarily have a significantly larger local Lipschitz constant in some cases.
	
	In the proof of Theorem~\ref{thm:approximation_inverse_function} of our work, we first considered an extended inversion function $f$ which we prove to be Lipschitz continuous. Then we approximate $f$ by a $\relu$ network $\tilde{f}$ such that $\|\tilde{f}(y) - f(y)\|_2 \leq \delta' \|y\|_2$. And indeed, even though $f$ is Lipschitz continuous and $\delta'$ can be arbitrarily small, we cannot conclude anything about the local Lipschitz constants of $\tilde{f}$ based on this method. Specifically, in the sparse recovery case, Corollary~\ref{cor:approximation_rip} states that for an exactly sparse signal $x$ with $\|x\|_2 = 1$,
	\begin{align*}
	    \|\tilde{f}(A x + e) - x\|_2 \leq \delta' + D \|e\|_2.
	\end{align*}
	Therefore, we obtain gradients
	\begin{align*}
	    \frac{\|\tilde{f}(A x + e) - \tilde{f}(A x)\|_2}{\|e\|_2} \leq \frac{2 \delta'}{\|e\|_2} + D.
	\end{align*}
	For very small $\|e\|_2$, specifically $\|e\|_2 \lesssim \delta'$, this becomes very large and therefore it becomes clear that our method cannot provide a bound to control the local Lipschitz constant of $\tilde{f}$. However, these large gradients only occur for very small $\|e\|_2$ and if specifically $\|e\|_2 \geq \delta'$ (recall that $\delta'$ can be chosen arbitrarily small), then the above gradient is bounded by
	\begin{align*}
	    \frac{\|\tilde{f}(A x + e) - \tilde{f}(A x)\|_2}{\|e\|_2} \leq 2 + D,
	\end{align*}
	i.e., a constant. So to summarize, the networks provided by our method might actually have very large local Lipschitz constants. However, these are only relevant for very small deviations and in this way, robust recovery as in Theorem~\ref{thm:approximation_inverse_function} is still possible.
	
	The results in Section~\ref{sec:optimization_networks} below also show that for a large class of minimization problems, neural networks can achieve the same robustness with respect to perturbations of size $\|e\|_2 \gtrsim \delta'$ even though the local Lipschitz constant might be significantly larger. Nevertheless, we can chose the $\delta'$ arbitrarily close to $0$.
	
	\section{Networks from Optimization Based Approaches} \label{sec:optimization_networks}
	
	Other than the neural network approaches of this work, classical compressed sensing studies optimization based methods to solve this problem, see Chapter 4 in \cite{comp_sen} for an overview. In contrast to the previous result in this work, which is independent of this, in this section we solve the sparse recovery problem by approximating the solution of an optimization problem with a neural network. In particular, we consider the $\ell_1$ minimization problem
	\begin{align}
	    & \min \|z\|_1 & \text{s.t. } & \|A z - y\|_2 \leq \eta,
	    \label{eq:l1_minimization_description}
	\end{align}
	for an $\eta \geq 0$, which can be shown to give a stable reconstruction of sparse $x$ from their measurements $y = A x$ for suitable measurement matrices $A$ \cite{comp_sen}.
	
	To show that \eqref{eq:l1_minimization_description} can be solved using a $\relu$ network, we need to show that the function that maps vectors $y$ to the corresponding minimizer in \eqref{eq:l1_minimization_description} is continuous. However, it is not clear that for each $y$, \eqref{eq:l1_minimization_description} has a unique solution. In fact, previous works such as \cite{zhang2016one} have shown uniqueness under certain circumstances but this might not be the case in general.
	
	This leads to the concept of multifunctions. Unlike a usual function $f: X \rightarrow Y$ that maps each $x \in X$ to exactly one $f(x) \in Y$, a multifunction maps each $x \in X$ to a subset of $Y$ while usually the empty set is excluded. Therefore, a multifunction can also be seen as a function $F: X \rightarrow 2^{Y} \backslash \{\emptyset\}$ with values in the power set $2^{Y}$ of $Y$. Using this concept, we can always describe the map from the vector $y$ to the set of minimizers of \eqref{eq:l1_minimization_description} as a multifunction if the minimization problem is feasible. We will show that this multifunction satisfies a generalization of continuity and that eventually there is a continuous selection function that maps every $y$ to an approximation of \textit{one} of the solutions of \eqref{eq:l1_minimization_description}.
	
	There exists an extensive theory about multifunctions and generalizations of well-known concepts of functions to them. This has been known as \textit{set-valued analysis} or \textit{multivalued analysis} and textbooks such as \cite{aubin2009set} and \cite{handbook_multivalued_analysis} can provide a detailed summary of this. In the following presentation of the most important concepts, we mostly use the notation and terminology of \cite{handbook_multivalued_analysis}.
	
	\begin{definition}
	    Let $X, Y$ be sets. A multifunction $F: X \rightarrow 2^{Y} \backslash \{\emptyset\}$ is a function that maps from $M$ to the power set $2^Y$ of $Y$ without the empty set.
	\end{definition}
	
	One important tool which we need in this section is the generalization of continuity to multifunctions. The following properties have also been known under the terms ''upper/lower \textit{hemi}continuous`` in the literature.
	
	\begin{definition}[Definition 2.3 / Remark 2.4 in \cite{handbook_multivalued_analysis}]
	    Let $F: X \rightarrow 2^{Y} \backslash \{\emptyset\}$ be a multifunction between Hausdorff topological spaces $X$ and $Y$. For $x_0 \in X$, we say that
	    \begin{itemize}
	        \item $F$ is upper semicontinuous at $x_0$ if for all open sets $V \subset Y$ with $F(x_0) \subset V$, there exists a neighborhood $U$ of $x_0$ such that for all $x \in U$, $F(x) \subset V$,
	        \item $F$ is lower semicontinuous at $x_0$ if for all open sets $V \subset Y$ with $F(x_0) \cap V \neq \emptyset$, there exists a neighborhood of $x_0$ such that for all $x \in U$, $F(x) \cap V \neq \emptyset$,
	        \item $F$ is continuous at $x_0$ if $F$ is upper semicontinuous at $x_0$ and lower semicontinuous at $x_0$.
	    \end{itemize}
	    We say that $F$ is (upper/lower semi-)continuous if it is (upper/lower semi-)continuous at all points $x_0 \in X$.
	\end{definition}
	
	For single-valued functions, i.e., multifunctions $F$ such that $|F(x)| = 1$ for all $x$ in the domain, all these terms coincide to the usual term of continuity of functions. A simple standard example for a multivalued function that is upper but not lower semicontinuous is given by $F_1: \R \rightarrow 2^{\R} \backslash \{\emptyset\}$, $F_1(x) = \{1\}$ if $x \neq 0$ and $F_1(0) = [0, 1]$. A lower but not upper semicontinuous function is given by $F_2: \R \rightarrow 2^{\R} \backslash \{\emptyset\}$, $F_2(x) = [0, 1]$ if $x \neq 0$ and $F_2(0) = \{0\}$. Example 2.8 in \cite{handbook_multivalued_analysis} gives more details and explanations about this. So in general, none of these two properties implies the other.

    In optimization problems such as \eqref{eq:l1_minimization_description}, the feasible region is a multifunction of the parameters (here $y$). Berge's maximum theorem states that if this feasible region is a continuous multifunction and the objective function is a continuous function, then the multifunction mapping the parameters to the set of optimal solutions is upper semicontinuous.

	\begin{theorem}[Berge's maximum theorem, Theorem 3.4 in \cite{handbook_multivalued_analysis}] \label{thm:maximum_theorem}
	    Let $u: X \times Y \rightarrow \R$ be a continuous function and $F: Y \rightarrow 2^{X} \backslash \{\emptyset\}$ a continuous multifunction with compact values. Consider the optimization problem
	    \begin{align*}
	        & \max_{x} u(x, y) & \text{s.t. } & x \in F(y).
	    \end{align*}
	    Let $S: Y \rightarrow 2^{X} \backslash \{\emptyset\}$ be the multifunction mapping each $y \in Y$ to the optimizers and $v: Y \rightarrow \R$ be the function mapping each $y$ to the optimal value.
	    
	    Then $S$ is upper continuous with compact values and $v$ is continuous.
	\end{theorem}
	
	For a multifunction $F: X \rightarrow 2^{Y} \backslash \{\emptyset\}$, a \textit{selection} is defined as a single-valued function $f: X \rightarrow Y$ such that $f(x) \in F(x)$ for all $x \in X$. A particular question that has been studied is when there exists a continuous selection of $F$. Michael's selection theorem (Theorem 4.6 in \cite{handbook_multivalued_analysis}) states that lower semicontinuity of $F$ is enough for this. However, the above Theorem~\ref{thm:maximum_theorem} can only guarantee upper semicontinuity which is not sufficient for a continuous selection (also see \cite{handbook_multivalued_analysis}). Nevertheless, upper semicontinuous multifunctions still allow approximate selections with arbitrarily small perturbations in the argument and in the function value. In the following theorem, we use the notation $F(M)$ with a multifunction $F: X \rightarrow 2^{Y}$ and a subset $M \subset X$ for $F(M) := \bigcup_{x \in M} F(x)$.
	
	\begin{theorem}[Theorem 4.42 in \cite{handbook_multivalued_analysis}] \label{thm:approx_selection}
	    Let $X$ be a metric space, $Y$ a Banach space, $W \subset X$ open, $K \subset W$ compact, $F: \bar{W} \rightarrow 2^{Y} \backslash \{\emptyset\}$ an upper semicontinuous multifunction with convex values, then for every $\epsilon > 0$, there is an open neighborhood $G_\epsilon$ of $K$ and a locally Lipschitz function $f_\epsilon: G_\epsilon \rightarrow \mathrm{conv} F(K)$ with finite dimensional range such that for every $x \in G_\epsilon$, $f_\epsilon(x) \in F(K \cap B_\epsilon(x)) + B_\epsilon(0)$.
	\end{theorem}
	
	Now in order to apply the aforementioned results to the particular problem of approximating the solution of \eqref{eq:l1_minimization_description} with a continuous function, the first step is to show that the feasible region of the problem is described by a continuous multifunction. We show this for a slight restriction of the feasible region which is compact but this will not change the eventual minimizer set.

    \begin{lemma} \label{lem:feasible_region_continuous}
    	Let $A \in \R^{k \times n}$ ($k \leq n$), $B \in \R^{k \times m}$, and $\eta \in [0, \infty)$. Let $\|\cdot\|$ be a norm on $\R^k$ and $g: \R^m \rightarrow [1, \infty)$ a continuous function.
    	
    	Define the multifunctions $F_1, F_2, F: \R^m \rightarrow 2^{\R^n} \backslash \{\emptyset\}$ by
    	\begin{align*}
    		F_1(y) & = \{x \in \R^n \,\big|\, \|A x + B y\| \leq \eta \}, &
    		F_2(y) & = \{x \in \R^n \,\big|\, \|P_{\ker(A)} x\|_2 \leq g(y)\}, &
    		F(y) & = F_1(y) \cap F_2(y),
    	\end{align*}
    	where $P_{\ker(A)}$ is the orthogonal projection onto the kernel of $A$.
    	
    	Furthermore, assume that
    	\begin{itemize}
    		\item $g(y) \geq \|A^\dagger B y\|_2$ for all $y \in \R^m$
    		\item $B(\R^m) \subset A(\R^n)$
    	\end{itemize}
    	
    	Then $F$ is well-defined (i.e., $F(y) \neq \emptyset$ for all $y \in \R^m$), continuous and has compact values.
    \end{lemma}
    
    \begin{proof}
    	The condition that the range of $B$ is contained in the range of $A$ implies that
    	\begin{align*}
    		A A^\dagger B = P_{A(\R^n)} B = B,
    	\end{align*}
    	where $P_{A(\R^n)}$ is the orthogonal projection onto the range of $A$.
    	
    	Furthermore, by the equivalence of all norms on $\R^k$, there exists a constant $C > 0$ such that $\|z\|_2 \leq C \|z\|$ holds for all $z \in \R^k$.
    	
    	\textbf{Step 1: $F$ is well-defined:}
    	Since $A A^\dagger B = B$, for each $y \in \R^m$, $\|A (- A^\dagger B y) + B y\| = 0$, such that $- A^\dagger B y \in F_1(y)$ and by assumption $\|- A^\dagger B y\|_2 \leq g(y)$ such that also $- A^\dagger B y \in F_2(y)$ and thus $F(y) \neq \emptyset$.
    	
    	\textbf{Step 2: $F$ has compact values:} For each $y$, $F_1(y)$ and $F_2(y)$ are closed such that $F(y)$ is closed. Furthermore, note that $A^\dagger A = P_{\ker(A)^\bot}$. So for any $y \in \R^m$, $x \in F(y)$, we obtain
    	\begin{align*}
    		\|x\|_2 & \leq \|P_{\ker(A)} x\|_2 + \|A^\dagger A x\|_2
    		\leq g(y) + \|A^\dagger(A x + B y)\|_2 + \|A^\dagger B y\|_2 \\
    		& \leq g(y) + \|A^\dagger\|_{2 \rightarrow 2} \|A x + B y\|_2 + \|A^\dagger B y\|_2  \\
    		& \leq g(y) + \|A^\dagger\|_{2 \rightarrow 2} C \eta + \|A^\dagger B y\|_2.
    	\end{align*}
    	The right hand side does not depend on $x$ and therefore, $F(y)$ is also bounded and thus compact.

    	\textbf{Step 3: $F$ is lower semicontinuous}
    	
    	Let $y_0 \in \R^m$ and $V \subset \R^n$ open such that there exists an $x_0 \in F(y_0) \cap V$. Since $V$ is open, there exists a radius $\epsilon > 0$ such that $B_\epsilon(x_0) \subset V$.
    	
    	Since $g$ is continuous, there is a $\tilde{\delta} > 0$ such that for all $y \in B_{\tilde{\delta}}(y_0)$, $|g(y) - g(y_0)| < \frac{\epsilon}{4}$.
    	Now choose $\delta := \min\{ \tilde{\delta}, \frac{\epsilon}{4 \|A^\dagger B\|_{2 \rightarrow 2}} \} > 0$ if $A^\dagger B \neq 0$ and $\delta = \tilde{\delta}$ otherwise. Let $y \in B_\delta(y_0)$. We define the number $\lambda \in (0, 1]$ by
    	\begin{align*}
    		\lambda := \begin{cases}
    			1 & \text{if } \|P_{\ker(A)} x_0\|_2 \leq \frac{\epsilon}{4} \\
    			\frac{\epsilon}{4 \|P_{\ker(A)} x_0\|_2} & \text{otherwise},
    		\end{cases}
    	\end{align*}
    	such that we always have $1 - \lambda \geq 0$ and $\lambda \|P_{\ker(A)} x_0\|_2 \leq \frac{\epsilon}{4}$. Then we define
    	\begin{align} \label{eq:lower_semicont_proof_x}
    		x := x_0 + A^\dagger B (y_0 - y) - \lambda P_{\ker(A)} x_0.
    	\end{align}
    	
    	We observe $\|A x + B y\| = \|A x_0 + A A^\dagger B(y_0 - y) - 0 + B y\| = \|A x_0 + B y\| \leq \eta$, so $x \in F_1(y)$. Furthermore,
    	\begin{align*}
    		\|P_{\ker(A)} x\|_2 & = \|(1 - \lambda) P_{\ker(A)} x_0\|_2 \\
    		& = \|P_{\ker(A)} x_0\|_2 - \lambda \|P_{\ker(A)} x_0\|_2
    		\begin{cases}
    			= 0 \leq g(y) & \text{if } \lambda = 1 \\
    			\leq g(y_0) - \frac{\epsilon}{4} < g(y) & \text{otherwise},
    		\end{cases}
    	\end{align*}
    	showing that $x \in F_2(y)$, i.e., $x \in F(y)$. We also obtain
    	\begin{align*}
    		\|x - x_0\|_2 = \|A^\dagger B(y_0 - y) - \lambda P_{\ker(A)} x_0\|_2 \leq \|A^\dagger B\|_{2 \rightarrow 2} \delta + \lambda \|P_{\ker(A)} x_0\|_2 \leq \frac{\epsilon}{4} + \frac{\epsilon}{4} < \epsilon,
    	\end{align*}
    	this implies $x \in B_\epsilon(x_0) \subset V$. Therefore, $F(y) \cap V \neq \emptyset$ for any $y \in B_\delta(y_0)$. This shows that $F$ is lower semicontinuous.
    	
    	\textbf{Step 4: $F$ is upper semicontinuous}
    	
    	Let $y_0 \in \R^m$ and take an open set $V \subset \R^n$ such that $F(y_0) \subset V$.
    	
    	For points $y \in \R^m$, define the distance $d(y, F(y_0)) = \min_{y' \in F(y_0)} \|y - y'\|_2$. Since $F(y_0)$ is compact, this minimum always exists. 
    	
    	Assume that for each integer $k \geq 1$, there is a $y_k \in \R^m \backslash V$ such that $d(y_k, F(y_0)) \leq \frac{1}{k}$. Then $(y_k)$ forms a sequence in the compact set $F(y_0) + \bar{B_1}(0)$. Therefore, it has a convergent subsequence $(y_{k_l})$ with limit $\bar{y}$. By continuity, $d(\bar{y}, F(y_0)) = 0$ such that by compactness $\bar{y} \in F(y_0)$. On the other hand, $\R^m \backslash V$ is closed such that $\bar{y} \in \R^m \backslash V$. This contradicts the assumption that $F(y_0) \subset V$. Therefore, there is a radius $\epsilon > 0$ such that all $y \in \R^m$ with $d(y, F(y_0)) < \epsilon$ belong to $V$, i.e., $F(y_0) + B_\epsilon(0) \subset V$.
    	
    	The rest of the argument is similar to the proof of lower semicontinuity. Since $g$ is continuous, there is a $\tilde{\delta} > 0$ such that for all $y \in B_{\tilde{\delta}}(y_0)$, $|g(y) - g(y_0)| < \frac{\epsilon}{4}$.
    	Now choose $\delta := \min\{ \tilde{\delta}, \frac{\epsilon}{4 \|A^\dagger B\|_{2 \rightarrow 2}} \} > 0$ if $A^\dagger B \neq 0$ and $\delta = \tilde{\delta}$ otherwise. Let $y \in B_\delta(y_0)$ and take any $x \in F(y)$. We define the number $\lambda \in (0, 1]$ by
    	\begin{align*}
    		\lambda := \begin{cases}
    			1 & \text{if } \|P_{\ker(A)} x\|_2 \leq \frac{\epsilon}{4} \\
    			\frac{\epsilon}{4 \|P_{\ker(A)} x\|_2} & \text{otherwise},
    		\end{cases}
    	\end{align*}
    	such that we always have $1 - \lambda \geq 0$ and $\lambda \|P_{\ker(A)} x\|_2 \leq \frac{\epsilon}{4}$. Then we define
    	\begin{align*}
    		\bar{x} := x + A^\dagger B (y - y_0) - \lambda P_{\ker(A)} x.
    	\end{align*}
    	Note that this definition is analogous to \eqref{eq:lower_semicont_proof_x} in the proof of the lower semicontinuity. Therefore, we can follow the same subsequent steps and prove that $\bar{x} \in F(y_0)$ and $\|\bar{x} - x\|_2 < \epsilon$.

    	This implies $x \in F(y_0) + B_\epsilon(0) \subset V$. Since this holds for any $x \in F(y)$, $F(y) \subset V$ for any $y \in B_\delta(y_0)$, proving that $F$ is upper semicontinuous.
    \end{proof}

	Now, using the continuity of the feasible region from above, we can use the tools from multivalued analysis to show that for a certain class of minimization problems, there exists a continuous function whose values are approximate optimal solutions.
	
	\begin{lemma} \label{lem:optimization_cont_selection}
	    Take $\tilde{A} \in \R^{k \times n}$, $\tilde{B} \in \R^{k \times m}$, $\eta \in [0, \infty)$, a norm $\|\cdot\|$ on $\R^k$, and a continuous function $u: \R^n \times \R^m \rightarrow \R$. Consider the optimization problem
	    \begin{align} \label{eq:minimization_general_1}
	        & \min_{z \in \R^n} u(z, y) & \text{s.t. } & \|\tilde{A} z + \tilde{B} y\| \leq \eta,
	    \end{align}
	    where $\tilde{B}(\R^m) \subset \tilde{A}(\R^n)$. Furthermore, assume that there exists a continuous function $\bar{g}: \R^m \rightarrow \R$ and a coefficient $\alpha \in (0, \infty)$, such that for all $y \in \R^m$ and feasible $z \in \R^n$,
	    \begin{align}
	        \|z\|_2 \leq \alpha u(z, y) + \bar{g}(y).
	        \label{eq:condition_z_2_upper_bound}
	    \end{align}
	    
	    Let $\|\cdot\|_{I}$ be a norm on $\R^m$.
	    
	    For each $\epsilon > 0$ and each compact $V \subset \R^m$, there is a function $\tilde{f}: V \rightarrow \R^n$, represented by a $\relu$ network with one hidden layer, such that for all $y \in V$, there is a $\tilde{y} \in V$ and a solution $\tilde{x} \in \R^n$ of \eqref{eq:minimization_general_1} for $\tilde{y}$, such that $\|y - \tilde{y}\|_2 < \epsilon$ and $\|\tilde{f}(y) - \tilde{x}\|_{I} < \epsilon$.
	\end{lemma}
	
	\begin{proof}
	    First we define the continuous function $g(y) := \bar{g}(y) + \alpha u(- \tilde{A}^\dagger \tilde{B} y, y)$ and consider the corresponding multifunction $F$ defined in Lemma~\ref{lem:feasible_region_continuous} (with the matrices $\tilde{A}$ and $\tilde{B}$). Then every minimizer of
	    \begin{align} \label{eq:minimization_general_modified_1}
	        & \min u(z, y) & \text{s.t. } & z \in F(y)
	    \end{align}
	    also minimizes \eqref{eq:minimization_general_1}.
	    
	    Assume that this is not the case and there is a minimizer $\tilde{x}$ of \eqref{eq:minimization_general_modified_1} that does not minimize \eqref{eq:minimization_general_1}. Then there is an optimal solution $\hat{x}$ to \eqref{eq:minimization_general_1} (because of $\tilde{B}(\R^m) \subset \tilde{A}(\R^n)$, it is always feasible) with $u(\hat{x}, y) < u(\tilde{x}, y)$. So $\hat{x}$ cannot be feasible for \eqref{eq:minimization_general_modified_1}, i.e., $\hat{x} \notin F(y)$. On the other hand, $-\tilde{A}^\dagger \tilde{B} y$ is feasible for \eqref{eq:minimization_general_modified_1}. This implies
	    \begin{align*}
	        \|P_{\ker(\tilde{A})} \hat{x} \|_2 \leq \| \hat{x} \|_2 \leq \alpha u(\hat{x}, y) + \bar{g}(y) \leq \alpha u(\tilde{x}, y) + \bar{g}(y) \leq \alpha u(- \tilde{A}^\dagger \tilde{B} y, y) + \bar{g}(y) = g(y)
	    \end{align*}
	    and therefore $\hat{x}$ is also feasible for \eqref{eq:minimization_general_modified_1}, which contradicts the above observation.

	    Now by Lemma~\ref{lem:feasible_region_continuous}, $F$ is continuous with compact values such that by Theorem~\ref{thm:maximum_theorem}, the solution multifunction of \eqref{eq:minimization_general_modified_1}, $S: \R^m \rightarrow 2^{\R^n} \backslash \{\emptyset\}$ is upper semicontinuous with compact values and the optimal value function $v: \R^m \rightarrow \R$ is continuous.
	    
	    We apply the approximate selection Theorem~\ref{thm:approx_selection}. As a domain, consider the metric space $V$ endowed with the $\|\cdot\|_{2}$ norm. Then $V$ is an open and compact subset of itself. The space $\R^n$ endowed with the $\|\cdot\|_{I}$ norm is finite-dimensional and therefore a Banach space. We have shown that $S: \R^m \rightarrow 2^{\R^n} \backslash \{\emptyset\}$ is upper semicontinuous with convex values and this remains the case if we restrict $S$ to the metric space $V$ (whose topology is the subspace topology of $\R^m$). Therefore by Theorem~\ref{thm:approx_selection}, for every $\epsilon > 0$, there exists a continuous function $f: V \rightarrow \R^n$ such that for every $y \in V$, $f(y) \in S(B_\epsilon(y)) + B_\epsilon(0)$. This means that there exists $\tilde{y} \in V$ and $\tilde{x} \in S(\tilde{y})$ such that $\|y - \tilde{y}\|_2 < \frac{\epsilon}{2}$ and $\|f(y) - \tilde{x}\|_{I} < \frac{\epsilon}{2}$.
	    
	    By the universal approximation theorem for compact sets (Theorem~\ref{thm:universal_approximation} for each coordinate of $f$), for each $\epsilon$, there exists a $\relu$ network with one hidden layer that represents $\tilde{f}: V \rightarrow \R^n$, such that for all $y \in V$, $\|\tilde{f}(y) - f(y)\|_{I} < \frac{\epsilon}{2}$. Then for all $y \in V$, there exists a $\tilde{y} \in V$ and $\tilde{x} \in S(\tilde{y})$ such that $\|y - \tilde{y}\|_2 < \frac{\epsilon}{2} < \epsilon$ and $\|\tilde{f}(y) - \tilde{x}\|_{I} \leq \|\tilde{f}(y) - f(y)\|_{I} + \|f(y) - \tilde{x}\|_{I} < \epsilon$.
	\end{proof}
	
	\begin{remark} \label{rem:minimization_examples}
	    Lemma~\ref{lem:optimization_cont_selection} can be applied to the following optimization problems that have been used to solve the sparse recovery problem, i.e., recovering $x$ from $y = A x + e$. An overview with a detailed explanation of the following techniques can be found in Section~3.1 in \cite{comp_sen}.
	    \begin{itemize}
	        \item Quadratically constrained basis pursuit:
	        \begin{align}
	            &\min_{z \in \R^n} \|z\|_1 & \text{s.t. } & \|A z - y\|_2 \leq \eta
	            \label{eq:qc_basis_pursuit}
	        \end{align}
	        for $\|e\|_2 \leq \eta$. Here $u(z, y) = \|z\|_1$ is continuous, $\tilde{A} = A$, and $\tilde{B} = Id_m$. So Lemma~\ref{lem:optimization_cont_selection} can be applied if $\rank(A) = m$. If this is not the case and $\rank(A) = m' < m$, we can replace $A$ by $P A \in \R^{m' \times n}$ where $P \in \R^{m' \times m}$ is a bijective and orthogonal map from $A(\R^n)$ to $\R^{m'}$. Then $P A$ satisfies the same RIP as $A$.
	        
	        Note that also the condition \eqref{eq:condition_z_2_upper_bound} is fulfilled since for all feasible $z$, $\|z\|_2 \leq \|z\|_1 = 1 \cdot u(z, y) + 0$.
	        
	        \item Basis pursuit denoising:
	        \begin{align}
	            &\min_{z \in \R^n} \lambda \|z\|_1 + \|A z - y\|_2^2
	            \label{eq:den_basis_pursuit}
	        \end{align}
	        for a parameter $\lambda > 0$. For each feasible $z$, $\|z\|_2 \leq \lambda u(z, y)$ such that \eqref{eq:condition_z_2_upper_bound} is fulfilled for $\alpha = \lambda$. Again the objective function is continuous and we can apply Lemma~\ref{lem:optimization_cont_selection} for $\tilde{A} = 0 \in \R^{1 \times n}$ and $\tilde{B} = 0 \in \R^{1 \times m}$.
	        
	        \item LASSO:
	        \begin{align}
	            & \min_{z \in \R^n} \|A z - y\|_2 & \text{s.t. } & \|z\|_1 \leq \tau
	            \label{eq:lasso}
	        \end{align}
	        for a parameter $\tau \geq 0$.\eqref{eq:condition_z_2_upper_bound} is fulfilled since for all feasible $z$, $\|z\|_2 \leq \|z\|_1 \leq u(z, y) + \tau = u(z, y) + \bar{g}(y)$ for the continuous function $\bar{g}(y) = \tau$. Lemma~\ref{lem:optimization_cont_selection} can be applied again for $\tilde{A} = Id_n$ and $\tilde{B} = 0 \in \R^{n \times m}$.
	        
	        \item Dantzig selector:
	        \begin{align}
	            & \min_{z \in \R^n} \|z\|_1 & \text{s.t. } & \|A^*(A z - y)\|_\infty \leq \eta
	            \label{eq:dantzig}
	        \end{align}
	        for a parameter $\tau \geq 0$ for $\|A^* e\|_\infty \leq \eta$. Here $\tilde{A} = A^* A$, $\tilde{B} = A^*$ and therefore $\tilde{A}(\R^n) = A^*(\R^m) = \tilde{B}(\R^m)$, so Lemma~\ref{lem:optimization_cont_selection} can be applied again. \eqref{eq:condition_z_2_upper_bound} is fulfilled for the same reason as in \eqref{eq:qc_basis_pursuit}.
	    \end{itemize}
	\end{remark}

    \begin{remark}
    	Lemma~\ref{lem:optimization_cont_selection} provides an approximate selection of solutions of the optimization problem \eqref{eq:minimization_general_1}, i.e., $\tilde{f}(y)$ is close to an optimal solution of \eqref{eq:minimization_general_1} for a parameter that is close to $y$. To show this, we used the approximate selection Theorem~\ref{thm:approx_selection}.
    	
    	One might wonder whether there exists a continuous exact selection, i.e., a continuous function $f: \R^m \rightarrow \R^n$ such that for each $y \in \R^m$, $f(y)$ is exactly a solution of \eqref{eq:minimization_general_1} for the parameter $y$. Indeed, in the field of multivalued analysis, Michael's selection theorem (Theorem~4.6 in \cite{handbook_multivalued_analysis}) can guarantee the existence of a continuous selection of a multifunction. It requires this multifunction (i.e. the multifunction of solutions of \eqref{eq:minimization_general_1} in our application) to be lower semicontinuous. However, Berge's maximum theorem (Theorem~\ref{thm:maximum_theorem}) can only guarantee upper semicontinuity for the solution function.
    	
    	Indeed, the following example shows that actually not in all cases in which Lemma~\ref{lem:optimization_cont_selection} can be applied, an exact continuous selection exists. Consider the continuous function $u: \R^2 \times \R^2 \rightarrow \R$
    	\begin{align*}
    		& u(z, y) = \left\| \begin{pmatrix} y_1 & 0 \\ 0 & y_2 \end{pmatrix} z \right\|_1 + \max\{2, \|z\|_2\}
    	\end{align*}
    	and the minimization problem
    	\begin{align} \label{eq:minimization_counterexample}
    		& \min_{z \in \R^2} u(z, y) & \text{s.t. } | \begin{pmatrix} 1 & 1 \end{pmatrix} z - \begin{pmatrix} 1 & 0 \end{pmatrix} y | \leq 0,
    	\end{align}
    	which satisfies the requirements of Lemma~\ref{lem:optimization_cont_selection} (including \eqref{eq:condition_z_2_upper_bound}). 
    	
    	We are interested in the case $y \in \{1\} \times (0, 2)$. The $\max\{2, \|z\|_2\}$ condition is only required to ensure \eqref{eq:condition_z_2_upper_bound} but it will not change the optimal solutions in this case. To see this, consider the minimization problem without the $\max\{2, \|z\|_2\}$, i.e., with $u(z, y)$ replaced by $\tilde{u}(z, y) = u(z, y) - 2\max\{2, \|z\|_2\}$ for $y \in \{1\} \times (0, 2)$. This becomes
    	\begin{align*}
    		& \min_{z \in \R^2} |z_1| + |y_2 z_2| & \text{s.t. } & z_1 + z_2 = 1,
    	\end{align*} 
    	which is equivalent to
    	\begin{align*}
    		&\min_{z_1 \in \R} |z_1| + |y_2(1 - z_1)|.
    	\end{align*}
    	We obtain the following optimal values and sets of all optimal solutions depending on $y_2$:
    	\begin{itemize}
    		\item $y_2 \in (0, 1)$: minimum $y_2$ at $z_1 = 0$
    		\item $y_2 = 1$: minimum $y_2 = 1$ at $z_1 \in [0, 1]$
    		\item $y_2 \in (1, 2)$: minimum $1$ at $z_1 = 1$.
    	\end{itemize}
    	All the solutions have $z_1 \in [0, 1]$ and therefore $z_2 \in [0, 1]$, $\|z\|_2 \leq \sqrt{2} < 2$. This shows that adding $\max\{2, \|z\|_2\}$ to the objective function will not change the set of minimizers and the solution sets from above are also the solution sets of \eqref{eq:minimization_counterexample}.
    	
    	So if there is a continuous function $f: \R^2 \rightarrow \R^2$, such that for every $y \in \R^2$, $f(y)$ is an optimal solution of \eqref{eq:minimization_counterexample}, we would need to have
    	\begin{align*}
    		(f(1, y_2))_1 & = 0 & \text{for all } y_2 & \in (0, 1) \\
    		(f(1, y_2))_1 & = 1 & \text{for all } y_2 & \in (1, 2).
    	\end{align*}
    	However, in this way $f$ cannot be continuous at the point $(1, 1)$.

    	Nevertheless, there might still be exact continous selections for some of the most important applications listed in Remark~\ref{rem:minimization_examples}. The work in \cite{bringmann2018homotopy} considers problem \eqref{eq:den_basis_pursuit} and shows that there is an optimal solution that continuously depends on the parameter $\lambda$. With similar techniques, it might also be possible to show that there is an optimal solution that continuously depends on $y$. However, the above counterexample shows that this is not always possible in the generalized setting of Lemma~\ref{lem:optimization_cont_selection}. Furthermore, since we approximate the solution functions using the universal approximation theorem with an arbitrary but positive precision $\delta > 0$, having an exact selection would not lead to any essential improvement anyway.
    \end{remark}
	
	The above Lemma~\ref{lem:optimization_cont_selection} concerns compact domains (and therefore only one hidden layer). The following Theorem turns this into a positive homogeneous version that enables results similar to the previous parts of this work.

    \begin{theorem} \label{thm:guarantee_from_minimization}
    	Let $A \in \R^{m \times n}$, $U \subset \R^n$ positive homogeneous. Consider matrices $\tilde{A} \in \R^{k \times n}$, $\tilde{B} \in \R^{k \times m}$ with $\tilde{B}(\R^m) \subset \tilde{A}(\R^n)$, a continuous $u: \R^n \times \R^m \rightarrow \R$, $\eta \in [0, 1]$ and a norm $\|\cdot\|$ on $\R^k$. We define the minimization problem
    	\begin{align} \label{eq:minimization_general}
    		& \min_{z \in \R^n} u(z, y) & \text{s.t. } \|\tilde{A} z + \tilde{B} y\| \leq \eta.
    	\end{align}
    	
    	Furthermore, assume that there exists a continuous function $\bar{g}: \R^m \rightarrow \R$ and a coefficient $\alpha \in (0, \infty)$, such that for all $y \in \R^m$ and feasible $z \in \R^n$,
    	\begin{align*}
    		\|z\|_2 \leq \alpha u(z, y) + \bar{g}(y).
    	\end{align*}
    	
    	Let $\|\cdot\|_{I}$ be a norm on $\R^n$ and $\|\cdot\|_{II}$ a norm on $\R^m$.
    	
    	Assume that for each $x \in \R^n$ and $e \in \R^m$, $\|e\|_{II} \leq \eta$, any optimal solution $\hat{x}$ of \eqref{eq:minimization_general} for $y = A x + e$ satisfies
    	\begin{align*}
    		\|\hat{x} - x\|_{I} \leq v(x, \eta),
    	\end{align*}
    	where $v: \R^n \times \R \rightarrow [0, \infty)$ satisfies $v(\lambda x, \lambda \eta) = \lambda v(x, \eta)$ for all $\lambda \geq 0$, $x \in \R^n$, $\eta \in \R$ and $\eta \mapsto v(x, \eta)$ is increasing for each $x$. Assume that this also holds for $\eta = 0$.
    	
    	Then for each $\delta > 0$, there exists a function $\tilde{f}: \R^m \rightarrow \R^n$, represented by a $\relu$ network with two hidden layers, such that for all $x \in \R^n$, $e \in \R^m$, $\|e\|_{II} \leq \frac{\eta}{3} \|A x\|_{II}$, $y = A x + e$,
    	\begin{align*}
    		\|\tilde{f}(y) - x\|_{I} \leq \delta \|x\|_2 + v(x, \frac{4}{3} \eta \|A x\|_{II}).
    	\end{align*}
    \end{theorem}
    
    \begin{proof}
    	Consider the unit sphere of the $\|\cdot\|_{II}$ norm
    	\begin{align*}
    		S_{II} := \{x \in \R^m \,\big|\, \|x\|_{II} = 1\} \subset \R^m.
    	\end{align*}
    	
    	$S_{II}$ is a compact set and therefore we can apply Lemma~\ref{lem:optimization_cont_selection} to obtain that for each $\epsilon > 0$, there is a continuous function $f: S_{II} \rightarrow \R^n$ such that for all $y \in S_{II}$, there exists a $\tilde{y} \in S_{II}$ and a solution $\tilde{x} \in \R^n$ of \eqref{eq:minimization_general} for $\tilde{y}$ such that $\|y - \tilde{y}\|_2 < \epsilon$ and $\|\tilde{f}(y) - \tilde{x}\|_{I} < \epsilon$.
    	
    	$f$ is defined on $S_{II}$ such that we can extend it to a positive homogeneous, continuous function $f: \R^m \rightarrow \R^n$ on the entire space. Now take any $x \in \R^n$ and $e \in \R^m$ with $\|e\|_{II} \leq \frac{\eta}{3} \|A x\|_{II}$. Let $y = A x + e$. Then
    	\begin{align} \label{eq:bound_y_Ax_II}
    		\frac{2}{3} \|A x\|_{II} \leq (1 - \frac{\eta}{3}) \|A x\|_{II} \leq \|y\|_{II} \leq (1 + \frac{\eta}{2}) \|A x\|_{II} \leq \frac{4}{3} \|A x\|_{II}.
    	\end{align}
    	
    	Assume $y \neq 0$ for now. Define
    	\begin{align*}
    		\bar{x} & = \frac{x}{\|y\|_{II}} & \bar{e} & = \frac{e}{\|y\|_{II}} & \bar{y} = A \bar{x} + \bar{e} = \frac{y}{\|y\|_{II}},
    	\end{align*}
    	such that $\bar{y} = A \bar{x} + \bar{e}$. So $\bar{y} \in S_{II}$ and thus by the previous observation, there is a $\bar{y}' \in S_{II}$ and an optimal solution $\bar{x}' \in \R^n$ of \eqref{eq:minimization_general} for $\bar{y}'$ such that $\|\bar{y}' - \bar{y}\|_2 < \epsilon$ and $\|f(\bar{y}) - \bar{x}'\|_{I} < \epsilon$. 
    	
    	There is a constant $C > 0$ such that $\|w\|_{II} \leq C \|w\|_2$ for all $w \in \R^m$. We can choose $\epsilon \leq \frac{\eta}{2 C}$.
    	
    	Define $\bar{e}' := \bar{e} + \bar{y}' - \bar{y}$. Then $\bar{y}' = A \bar{x} + \bar{e}'$ and 
    	\begin{align*}
    		\|\bar{e}'\|_{II} & \leq \|\bar{e}\|_{II} + C \|\bar{y}' - \bar{y}\|_2
    		\leq \frac{\|e\|_{II}}{\|y\|_{II}} + C \epsilon
    		\leq \frac{\frac{\eta}{3} \|A x\|_{II}}{\frac{2}{3} \|A x\|_{II}} + C \epsilon
    		\leq \frac{\eta}{2} + \frac{\eta}{2} = \eta.
    	\end{align*}
    	
    	So since $\bar{x}'$ is an optimal solution of \eqref{eq:minimization_general} for $\bar{y}'$,
    	\begin{align*}
    		\|\bar{x}' - \bar{x}\|_{I} \leq v(\bar{x}, \eta).
    	\end{align*}
    	Since $\|f(\bar{y}) - \bar{x}'\|_{I} < \epsilon$,
    	\begin{align*}
    		\|f(\bar{y}) - \bar{x}\|_{I} \leq  \|f(\bar{y}) - \bar{x}'\|_{I} + \|\bar{x}' - \bar{x}\|_{I}
    		\leq \epsilon + v(\bar{x}, \eta).
    	\end{align*}
    	
    	Now recall that $x = \|y\|_{II} \bar{x}$ and we defined $f$ by a positive homogeneous extension such that in general,
    	\begin{align}
    		\|f(y) - x\|_{I} & = \|y\|_{II} \|f(\bar{y}) - \bar{x}\|_{I} \leq
    		\epsilon \|y\|_{II} + v(\frac{x}{\|y\|_{II}}, \eta) \|y\|_{II}
    		= \epsilon \|y\|_{II} + v(x, \eta \|y\|_{II}) \nonumber \\
    		& \leq \epsilon \frac{4}{3} \|A x\|_{II} + v(x, \frac{4}{3} \eta \|A x\|_{II})
    		\leq \epsilon \frac{4}{3} C \|A\|_{2 \rightarrow 2} \|x\|_2 + v(x, \frac{4}{3} \eta \|A x\|_{II}) \nonumber \\
    		& \leq \frac{\delta}{2} \|x\|_2 + v(x, \frac{4}{3} \eta \|A x\|_{II}),
    		\label{eq:optimization_bound_I_norm}
    	\end{align}
    	where the last step follows by choosing $\epsilon \leq \frac{3 \delta}{8 C \|A\|_{2 \rightarrow 2}}$.

    	It still remains to show \eqref{eq:optimization_bound_I_norm} for the case that $y = 0$. Then by \eqref{eq:bound_y_Ax_II}, also $A x = 0$ and therefore $e = 0$. By the assumption of the theorem, in this case $z = 0$ is feasible and thus optimal in \eqref{eq:minimization_general} for $\eta = 0$ and so
    	\begin{align*}
    		\|x - 0\|_I \leq v(x, \eta) = v(x, 0).
    	\end{align*}
    	Since we defined $f$ as a positive homogeneous extension, $f(y) = 0$ such that $\|f(y) - x\|_I = \|0 - x\|_I$ and \eqref{eq:optimization_bound_I_norm} also holds for $y = 0$.
    	
    	Now since $f$ is a continuous, positive homogeneous function, by Theorem~\ref{thm:uat_homogeneous} (applied to each component, together with equivalence of all norms), for each $\epsilon' > 0$, there is $\tilde{f}: \R^m \rightarrow \R^n$, represented by an unbiased $\relu$ network with two hidden layers, such that for all $y \in \R^m$,
    	\begin{align*}
    		\|\tilde{f}(y) - f(y)\|_{I} \leq \epsilon' \|y\|_{II}.
    	\end{align*}
    	
    	Then $\tilde{f}(0) = 0 = f(0)$ and for $y \neq 0$, by combining everything, we obtain for all $\epsilon' > 0$,
    	\begin{align*}
    		\|\tilde{f}(y) - x\|_I & \leq \|\tilde{f}(y) - f(y)\|_I + \|f(y) - x\|_I
    		\leq \epsilon' \|y\|_{II} + \frac{\delta}{2} \|x\|_2 + v(x, \frac{4}{3} \eta \|A x\|_{II}) \\
    		& \leq \frac{4}{3} \epsilon' \|A x\|_{II} + \frac{\delta}{2} \|x\|_2 + v(x, \frac{4}{3} \eta \|A x\|_{II}) 
    		\leq \delta \|x\|_2 + v(x, \frac{4}{3} \eta \|A x\|_{II})
    	\end{align*}
    	by choosing $\epsilon' \leq \frac{3 \delta}{8 C \|A\|_{2 \rightarrow 2}}$.
    	
    \end{proof}

	With Theorem~\ref{thm:guarantee_from_minimization}, we can construct a positive homogeneous network to solve an inverse problem that is known to be solved by a minimization problem. In particular, $\ell_1$ minimization has been studied for sparse recovery (see Chapter~4 of \cite{comp_sen}). Applying Theorem~\ref{thm:guarantee_from_minimization} to the quadratically constrained basis pursuit \eqref{eq:qc_basis_pursuit}, we obtain the following corollary.

    \begin{corollary} \label{cor:minim_rip}
    	Let $A \in \R^{m \times n}$ be a matrix of rank $m$, satisfying the $(2 s, \delta)$-restricted isometry property for a $\delta < 0.7$ and $\eta \in [0, \frac{1}{3}]$. Then for each $\delta' > 0$, there exists a function $\tilde{f}: \R^m \rightarrow \R^n$, represented by an unbiased $\relu$ network with two hidden layers, such that for all $x \in \R^n$, $e \in \R^m$ with $\|e\|_2 \leq \eta \|A x\|_2$, $p \in [1, 2]$,
    	\begin{align*}
    		\|\tilde{f}(A x + e) - x\|_p \leq \delta'\|x\|_2 + \frac{C}{s^{1 - 1/p}} \sigma_s(x)_1 + D s^{1/p - 1/2} \eta \|A x\|_2,
    	\end{align*}
    	where $\sigma_s(x)_1 := \inf_{x' \in \Sigma_s} \|x - x'\|_1$ and $C$, $D$ only depend on $\delta$.
    	
    	In particular, for $p = 1, 2$, we obtain
    	\begin{align*}
    		\|\tilde{f}(A x + e) - x\|_1 & \leq \delta'\|x\|_2 + C \sigma_s(x)_1 + D \sqrt{s} \eta \|A x\|_2 \\
    		\|\tilde{f}(A x + e) - x\|_2 & \leq \delta'\|x\|_2 + \frac{C}{\sqrt{s}} \sigma_s(x)_1 + D \eta \|A x\|_2.
    	\end{align*}
    \end{corollary}
    
    \begin{proof}
    	By Theorem 6.13 in \cite{comp_sen}, $A$ satisfies the $\ell_2$-robust null space property which in turn implies by Theorem 4.22 in \cite{comp_sen} that the solution $\hat{x}$ of \eqref{eq:qc_basis_pursuit} always satisfies
    	\begin{align*}
    		\|\hat{x} - x\|_p \leq \frac{C}{s^{1 - 1/p}} \sigma_s(x)_1 + D s^{1/p - 1/2} \eta
    	\end{align*}
    	for $p \in [1, 2]$.
    	
    	Then the result follows from Theorem~\ref{thm:guarantee_from_minimization} with $v(x, \eta) = \frac{C}{s^{1 - 1/p}} \sigma_s(x)_1 + D s^{1/p - 1/2} \eta$. 
    \end{proof}
	
	As shown in Proposition 3.2 in \cite{comp_sen}, the basis pursuit denoising \eqref{eq:den_basis_pursuit} and LASSO \eqref{eq:lasso} are as powerful as \eqref{eq:qc_basis_pursuit} since a solution of one of them can be shown to also optimize the other ones for suitable parameters.
	
    \begin{corollary} \label{cor:minim_rip_2}
    	Let $A \in \R^{m \times n}$ be a matrix of rank $m$ that satisfies the $(2 s, \delta)$-restricted isometry property for a $\delta < \frac{1}{3}$ and $\eta \in [0, \frac{1}{3}]$. Then for each $\delta' > 0$, there exists a function $\tilde{f}: \R^m \rightarrow \R^n$, represented by a $\relu$ network with two hidden layers, such that for all $x \in \R^n$, $e \in \R^m$ with $\|A^T e\|_\infty \leq \eta \|A^T A x\|_\infty$,
    	\begin{align*}
    		\|\tilde{f}(A x + e) - x \|_2 \leq \delta' \|x\|_2 + \frac{C}{\sqrt{s}} \sigma_s(x)_1 + D \eta \|A^T A x\|_\infty.
    	\end{align*}
    \end{corollary}
    
    \begin{proof}
    	Analogously to Corollary~\ref{cor:minim_rip}, this is a consequence of Theorem~\ref{thm:guarantee_from_minimization}, this time applied to the Dantzig selector \eqref{eq:dantzig}.
    	
    	As the only essential difference to Corollary~\ref{cor:minim_rip}, we need to ensure that $y \mapsto \|A^T y\|_\infty$ is a norm on $\R^m$. Clearly it fulfills all properties except the positive definiteness. The latter one is fulfilled if $A^T y \neq 0$ for all $y \neq 0$ which is equivalent to $\dim(\ker(A^T)) = 0$. This is fulfilled since $\dim(\ker(A^T)) = \dim((A(\R^n))^\bot) = m - \rank(A) = 0$.
    \end{proof}
	
    \begin{remark}
    	\begin{itemize}
    		\item Compared to the original minimization result, in Corollary~\ref{cor:minim_rip} (and analogously Corollary~\ref{cor:minim_rip_2}), the upper bound on the error, $\frac{\eta}{3}\|A x\|_2$, now depends on $\|A x\|_2$. This arises from making the solution positive homogeneous. However, the term that contributes to the deviation of the result is still equal to the maximal error up to constant factors.
    		
    		\item The condition $\rank(A) = m$ is satisfied for most interesting matrices, for example for Gaussian ones with probability $1$. If it is still not the case, we can replace $A$ by $P A$ for an orthogonal projection $P \in \R^{\rank(A) \times m}$ without changing its RIP.
    		
    		\item Compared to Corollary \ref{cor:approximation_rip}, Corollary~\ref{cor:minim_rip} provides deviation bounds in other norms and for $p = 2$ a better dependence on $\sigma_s(x)_1$. However, it requires an explicit bound on $\|e\|_2$ which influences the result while in Corollary~\ref{cor:approximation_rip} there is one network that works for all possible error levels.
    		
    		\item Another approach that allows for robust sparse recovery without a previously known bound on $\|e\|_2$ is given in \cite{wojtaszczyk2010stability} by a basis pursuit \eqref{eq:qc_basis_pursuit} with $\eta = 0$. However, to make this work, the measurement matrix $A$ must satisfy an additional condition beside the RIP which is known as the \textit{quotient property} with respect to a norm $\|\cdot\|$. Then the reconstruction error depends on $\|e\|$. This additional property holds with respect to the norm $\|\cdot\|_2$ for example for Gaussian matrices but not for Bernoulli matrices as shown in Section~11.3 in \cite{comp_sen}.
    	\end{itemize}
    \end{remark}

	\section{Discussion} \label{sec:nn_discussion}
	
	In this work, we have shown that $\relu$ networks with one hidden layer cannot even solve the sparse recovery problem for $1$-sparse vectors while in contrast with two hidden layers, they are capable of approximating this problem to an arbitrary precision and for arbitrary sparsity levels. The latter result can also be generalized to a larger class of inverse problems.
	
	A key assumption for these results is that we look at networks that take the positive homogeneous structure of the problem into account. This ensures the reconstruction to work for all possible signals without any bound on their norm.
	
	This also improved our understanding of how continuous positive homogeneous functions can be approximated with neural networks in general. We have seen that the $\relu$ function plays a unqiue role in their approximation and that the general approximation necessarily requires two layers.
	
	Despite showing that a good solution for the respective inverse problems is possible with rather shallow networks, our results of this work do not provide a statement about the width and efficiency of such networks. Possibly future research could use width-limited versions of the universal approximation theorem to investigate this question. For example, \cite{tang2020towards} shows such a statement for positive homogeneous networks (Theorem~2 in the supplement) which is based on Theorem~1 in \cite{lu2017expressive}. However, these results do not specify the depth of the network. Furthermore, future research could also search for guarantees regarding the training of the networks to solve inverse problems.

\section*{Acknowledgements}
This work was supported by the Deutsche Forschungsgemeinschaft (DFG, German Research Foundation) -- 456465471, 464123524, 273529854 and the authors also received funding by the German Federal Ministry of Education and Research and the Bavarian State Ministry for Science and the Arts.

	\printbibliography
	
\end{document}